\documentclass[sigconf,nonacm,10pt]{acmart}

\usepackage[english]{babel}
\usepackage{blindtext}
\usepackage{titlesec}
\titleformat{\section}
  {\normalfont\large\bfseries}{\thesection}{1em}{}
\titleformat{\subsection}
  {\normalfont\normalsize\bfseries}{\thesubsection}{1em}{}

\usepackage{color}
\usepackage{multirow}
\usepackage{url}
\usepackage{graphicx}
\usepackage{algorithmicx,algorithm}
\usepackage[noend]{algpseudocode}

\usepackage{enumitem} 
\usepackage{latexsym}
\usepackage{amsmath}
\usepackage{amsfonts}
\usepackage{tabularx}
\usepackage{balance}
\usepackage{subcaption}
\usepackage{caption}
\usepackage{enumitem} 
\usepackage{xspace}
\usepackage{booktabs}
\usepackage{tabularx}
\usepackage{balance}
\usepackage{graphicx}
\usepackage{svg}
\usepackage{pifont}
\usepackage[font=footnotesize]{caption}
\svgsetup{
    inkscapepath=i/svg-inkscape/
}
\svgpath{{svg/}}
\usepackage{etoolbox}
\usepackage[all]{nowidow}

\setlist{nolistsep}

\newcommand{\parabf}[1]{\medskip\noindent\textbf{#1}}
\newcommand{\parait}[1]{\medskip\noindent\textit{#1}}
\newcommand{\paraf}[1]{\noindent\textbf{#1}}
\newcommand{\cut}[1]{}

\newcommand{\sysname}{\textsc{Heddle}\xspace}

\usepackage{array}

\renewcommand\footnotetextcopyrightpermission[1]{} 
\acmConference{}{}{}
\acmBooktitle{}

\newtheorem{lemma}{Lemma}[section]

\begin{document}
\sloppy
\date{}
\title{
    \fontsize{15}{20}\selectfont
    \sysname: A Distributed Orchestration System for Agentic RL Rollout}

\author{
    \fontsize{11}{13}\selectfont
    \vspace{-0.0in}
    Zili Zhang$^{\ast}$\qquad Yinmin Zhong$^{\ast}$\qquad Chengxu Yang$^{\dagger}$\qquad Chao Jin$^{\ast}$\qquad \\
    \vspace{0.1in}
    Bingyang Wu$^{\ast}$\qquad Xinming Wei$^{\ast}$\qquad Yuliang Liu$^{\dagger}$\qquad Xin Jin$^{\ast}$\\
    \vspace{0.12in}
    $^{\ast}$\textit{School of Computer Science, Peking University} \qquad $^{\dagger}$\textit{Independent Researcher}\\
    \vspace{0.2in}
} 

\renewcommand{\shortauthors}{}
\renewcommand{\shorttitle}{}
\begin{sloppypar}

\begin{abstract}
Agentic Reinforcement Learning (RL) enables LLMs
to solve complex tasks by alternating between a data-collection rollout phase and a policy training phase.
During rollout, the agent generates trajectories, i.e., multi-step interactions between LLMs and external tools.
Yet, frequent tool calls induce long-tailed trajectory generation that bottlenecks rollouts.
This stems from \emph{step-centric} designs that ignore trajectory context,
triggering three system problems for long-tail trajectory generation: queueing delays, interference overhead, and inflated per-token time.
We propose \sysname, a \emph{trajectory-centric} system to optimize the \emph{when}, \emph{where}, and \emph{how} of agentic rollout execution.
\sysname integrates three core mechanisms: \ding{182} trajectory-level scheduling using
runtime prediction and progressive priority to minimize cumulative queueing;
\ding{183} trajectory-aware placement via presorted dynamic programming and opportunistic
migration during idle tool-call intervals to minimize interference; and \ding{184}
trajectory-adaptive resource manager that dynamically tunes model parallelism
to accelerate the per-token time of long-tail trajectories while maintaining high throughput for
short trajectories. Evaluations across diverse agentic RL workloads demonstrate that
\sysname effectively neutralizes the long-tail bottleneck, achieving up to 2.5$\times$ higher end-to-end rollout
throughput compared to state-of-the-art baselines. 
\end{abstract}

\maketitle

\section{Introduction}
\label{sec:introduction}

Agentic Reinforcement Learning (RL)~\cite{yao2022react,jin2025search,nakano2021webgpt,shinn2023reflexion,wang2023voyager}
is characterized by an iterative cycle of rollout (i.e., data collection) and training (i.e., policy optimization).
Moving beyond basic alignment~\cite{christiano2017deep,ouyang2022training} and static reasoning~\cite{guo2025deepseek, yu2025dapo}, agentic RL
enables LLMs to solve complex tasks through iterative, multi-step rollouts with external tool usage. This autonomy allows LLM agents to
navigate dynamic environments and continuously refine their strategies.
This paradigm has achieved notable success in frontier industry
products such as Claude Code~\cite{anthropic2025claudecode}, Deep Research~\cite{openai2025deepresearch}, and OpenClaw~\cite{openclaw}.

Despite the promise of agentic RL, its training efficiency is severely constrained by rollout phase. Unlike traditional LLM training that consumes static data,
agentic RL relies on online interactions where agents generate interleaved reasoning contexts and real-time tool execution
(i.e., multi-step agentic trajectories), as shown in Figure~\ref{fig:background:trajectory}.
Empirical analysis~\cite{qin2025seer, hu2025taming} identifies rollout generation as the dominant bottleneck,
consuming over 80\% of the entire training time.
This inefficiency is primarily driven by the severe straggler effect of rollout phase:
a long-tail distribution of generated agentic trajectories where a small fraction of complex, multi-step interactions
significantly prolongs the rollout makespan.
As shown in Figure~\ref{fig:background:distribution}, profiling Qwen3~\cite{yang2025qwen3} coding agentic rollouts on the CodeForces~\cite{li2022competition} dataset reveals that
both the number of generated tokens and the tool execution
latency are highly skewed. Consequently, as shown in Figure~\ref{fig:background:trajectory}, the majority of computational resources remain
idle waiting for these stragglers (i.e., long-tail trajectories) to complete, leading
to severe resource underutilization.

\begin{figure}[t!]
    \centering
    \includegraphics[width=\linewidth]{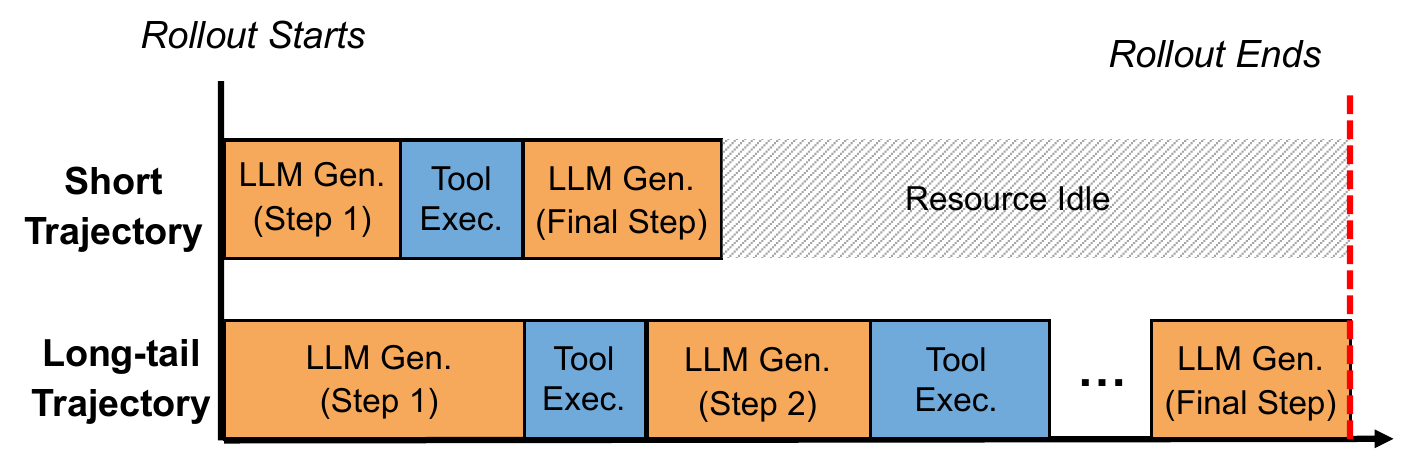}
    \vspace{-0.29in}
    \caption{Illustration of the agentic trajectory and the straggler effect.}
    \vspace{-0.26in}
    \label{fig:background:trajectory}
\end{figure}

To rigorously deconstruct this rollout bottleneck, we formulate the global rollout
makespan as being dictated by the slowest trajectory in the batch $\mathcal{T}$:
\begin{equation}
    T = \max_{i \in \mathcal{T}} \left( T_{\text{queue}}^{(i)} + N_{\text{tokens}}^{(i)} \cdot T \cdot \alpha^{(i)} + T_{\text{tool}}^{(i)} \right)
    \label{eq:makespan}
\end{equation}
While token count $N_{\text{tokens}}$ is algorithmically intrinsic and tool latency $T_{\text{tool}}$
is addressed by elastic serverless infrastructure~\cite{wang2018peeking,zhang2024jolteon},
the optimization burden falls on GPU-centric factors.
This formulation isolates three tractable components: queueing delay ($T_{\text{queue}}$),
the contention-driven interference coefficient ($\alpha$), and average
base per-token time ($T$).
These map to three core orchestration decisions:
\emph{when} to schedule the long-tail trajectory (minimize its queueing),
\emph{where} to place the long-tail trajectory (mitigate its compute and memory contention),
and \emph{how} to allocate resources for the long-tail trajectory (reduce its per-token time).

For scheduling, as shown in Figure~\ref{fig:background:framework},
existing agentic RL frameworks~\cite{verl,slime_github,zhong2025optimizing,zhong2025streamrl}
are fundamentally \emph{step-centric}, treating agentic steps as isolated requests.
This trajectory-agnostic design strips away critical metadata (e.g., ID, step index, length)
required for trajectory orchestration. Consequently, scheduling degenerates into a de facto
round-robin policy where multi-step trajectories repeatedly re-queue, inflicting severe
delays on long-tailed trajectories. For placement, existing
frameworks typically adopt either cache-affinity or least-load strategies.
Static cache-affinity triggers load imbalance due to lack of runtime migration,
while least-load policies incur prohibitive recomputation overhead and exacerbate interference for long-tailed trajectories
by redistributing trajectories per step.
Finally, for resource allocation, rigid and homogeneous GPU allocation fails to balance the high-throughput needs of
short trajectories with the low-latency requirements of long-tailed trajectories.
Ultimately, these isolated design choices compound to cause severe resource underutilization
during agentic RL rollout.

To this end, we propose \sysname, a distributed framework that rearchitects agentic RL rollout.
Moving beyond the step-centric design, \sysname adopts a \textit{trajectory-centric} design
to directly mitigate the execution stragglers.
Specifically, we introduce three synergistic techniques, e.g., trajectory-level scheduling,
trajectory-aware placement, and trajectory-adaptive resource management, to optimize the \textit{when},
\textit{where}, and \textit{how} of computation, respectively.

First, \emph{trajectory-level scheduling} dictates the scheduling (\emph{when}) of execution through
\emph{progressive priority scheduling} (\S\ref{sec:design:scheduler:progressive-priority-scheduling}). To bypass the pitfalls of round-robin,
\sysname assigns immediate execution precedence to long-tailed trajectories.
Since predicting trajectory length becomes more accurate as interactions unfold,
\sysname eschews static prioritization. Instead, it employs a runtime predictor (\S\ref{sec:design:scheduler:progressive-trajectory-prediction}) to
continuously refine length estimates after each step,
dynamically escalating the priority of long-tailed trajectories.
This progressive adjustment precisely accelerates true stragglers,
effectively minimizing their cumulative queueing delay ($T_{\text{queue}}$).

\begin{figure}[t!]
    \centering
    \includegraphics[width=\linewidth]{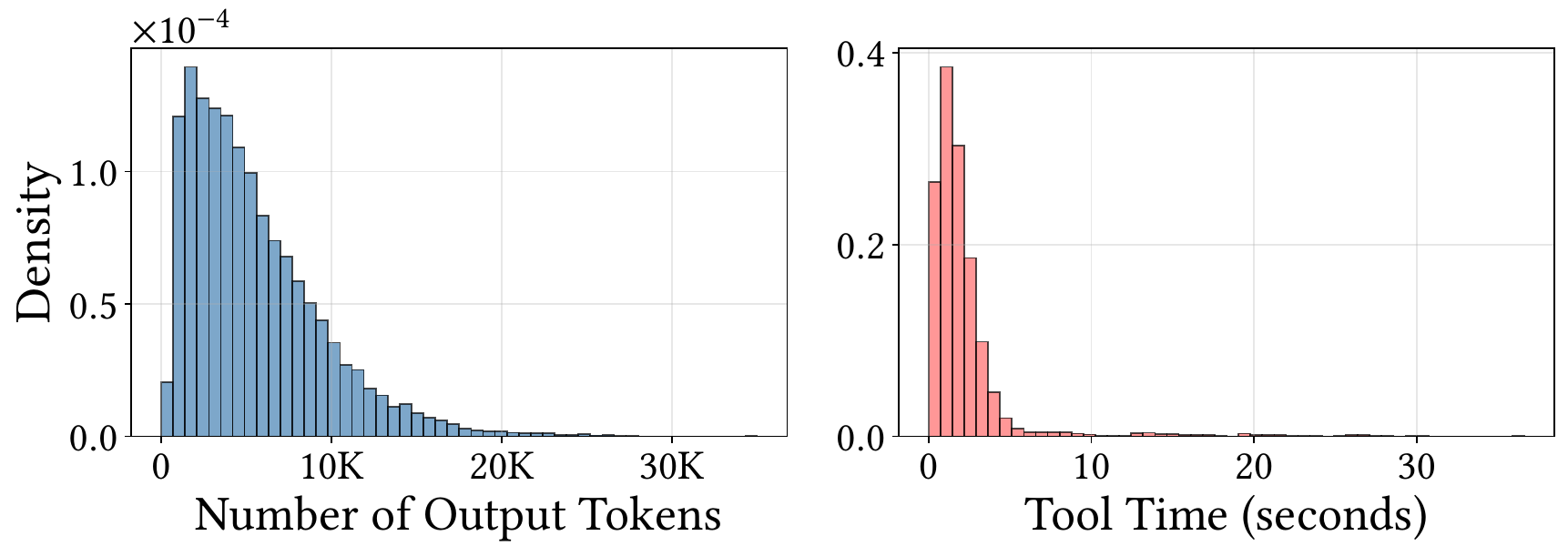}
    \vspace{-0.3in}
    \caption{Long-tailed distribution of coding agents.}
    \vspace{-0.26in}
    \label{fig:background:distribution}
\end{figure}

Second, \emph{trajectory-aware placement} dictates the spatial distribution (\emph{where})
of workloads via a two-phase strategy: \emph{presorted dynamic programming} (\S\ref{sec:design:placement:presorted-dp})
coupled with \emph{opportunistic runtime migration} (\S\ref{sec:design:placement:trajectory-migration}).
During initial dispatch, presorted dynamic programming sorts trajectories by predicted length,
computing an \emph{optimal} placement that minimizes the interference coefficient ($\alpha$)
of long-tailed trajectories.
Since prediction accuracy improves over time, \sysname incorporates runtime trajectory migration
to rectify initial placement deviations by dynamically migrating trajectory contexts and prefix caches.
Crucially, \sysname masks this migration overhead by transmitting data asynchronously
during tool-call intervals, keeping the critical execution path unblocked.
Ultimately, this dual approach minimizes the interference of long-tailed trajectories while preserving high data locality.

Third, \emph{trajectory-adaptive resource management} optimizes the resource allocation
problem (\emph{how}) to match specific trajectory characteristics. By breaking the rigid
constraint of homogeneous worker configurations, \sysname assigns distinct parallelism
strategies to different trajectories. It allocates low-latency, high-parallelism resources to
accelerate the base per-token time ($T$) of long-tail trajectories,
while assigning high-throughput, low-parallelism configurations to short trajectories.
To facilitate efficient online provisioning, we employ \emph{sort-initialized simulated annealing} (\S\ref{sec:design:resource-manager:simulated-annealing})
that rapidly converges on a near-optimal parallelism mapping for each trajectory.
This adaptive provisioning ensures that the execution of stragglers is
accelerated without compromising the aggregate system throughput for short trajectories.

In summary, we make the following contributions.
\begin{itemize}[leftmargin=*]
    \item We rigorously formulate the rollout makespan to deconstruct the system bottleneck.
    Our analysis isolates three critical performance bottlenecks for long-tail trajectories:
    \emph{queueing delay}, \emph{interference overhead}, and \emph{per-token time}.

    \item We propose \sysname, a distributed agentic RL system with a trajectory-centric design.
    By integrating trajectory-level scheduling, trajectory-aware placement, and
    trajectory-adaptive resource management, \sysname systematically optimizes \emph{when}, \emph{where}, and
    \emph{how} computation occurs, effectively neutralizing the long-tail bottleneck.

    \item We implement \sysname and evaluate it across diverse agentic RL workloads.
    The experimental results demonstrate that \sysname achieves up to 2.5$\times$
    higher throughput compared to state-of-the-art baselines.
\end{itemize}

\begin{figure}[t!]
    \centering
    \includegraphics[width=0.92\linewidth]{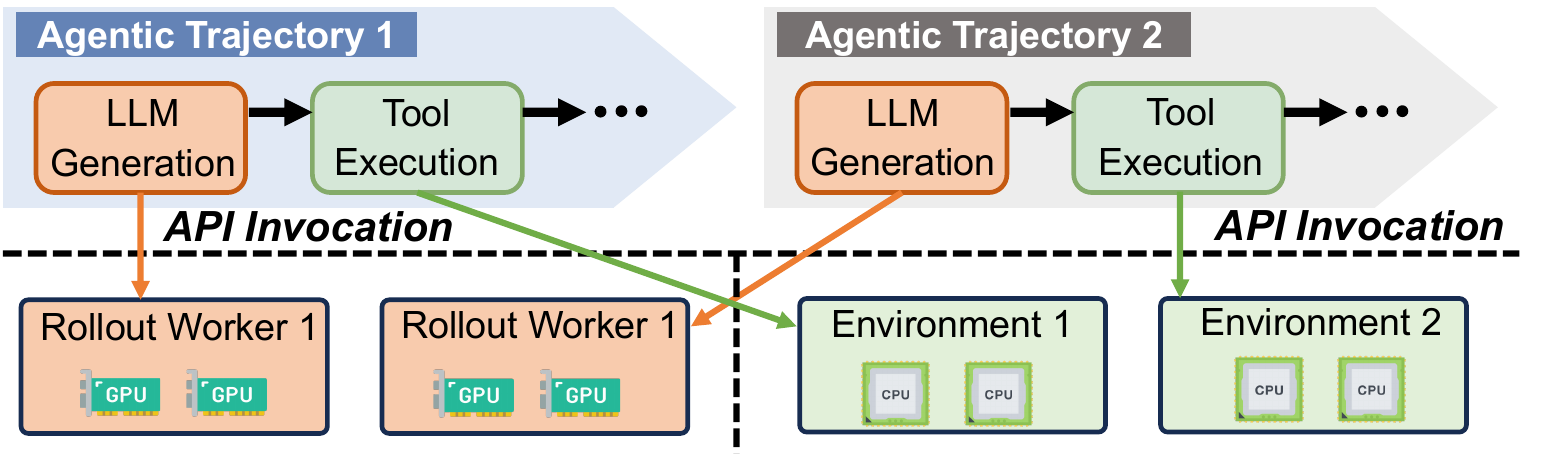}
    \vspace{-0.1in}
    \caption{Existing agentic RL framework.}
    \vspace{-0.2in}
    \label{fig:background:framework}
\end{figure}

\section{Background and Motivation}
\label{sec:background}
In this section, we introduce the background of agentic RL
and conduct a comprehensive analysis of the rollout bottleneck.
Then, we summarize the systematic challenges, which motivate the design of \sysname.

\subsection{Agentic RL}
\label{sec:background:agentic-rl}

Agentic Reinforcement Learning (agentic RL)~\cite{yao2022react,jin2025search,nakano2021webgpt,shinn2023reflexion,wang2023voyager} elevates LLMs into autonomous
agents capable of dynamic environmental interaction.
Departing from static reasoning tasks~\cite{guo2025deepseek, yu2025dapo, seed-thinking}, agentic RL empowers LLMs to
synthesize high-level plans and execute decisions within dynamic
environments. In this paradigm, during rollout,
an LLM agent generates multi-step trajectories: at each step, it generates logical reasoning,
invokes an external tool, and observes the resulting state transition to calibrate its next action.
This iterative loop continues until a terminal state is reached.
In the training phase, the accumulated interaction trajectories are used to optimize the agent's policy via
algorithms like PPO~\cite{schulman2017proximal} or GRPO~\cite{shao2024deepseekmath}, effectively grounding the
LLM's high-level planning in real-world execution.

For instance, a coding agent~\cite{yang2024swe,wang2024openhands,dai2026cuda} processing a high-level programming requirement generates a
multi-step trajectory. In its initial step,
the agent synthesizes a plan. Following the plan, it then retrieves relevant codebase context,
generates candidate snippets, and triggers validation tests. It then integrates
the resulting feedback into its context to iteratively debug errors across
subsequent steps until all tests pass.
As Figure~\ref{fig:background:trajectory} depicts, this looping process
transforms a standard single-shot inference request into a long-running agentic trajectory
characterized by the complex interleaving of LLM generation and external tool execution.

Beyond traditional chatbots, agentic RL underpins production-grade systems like
Claude Code~\cite{anthropic2025claudecode}, Deep Research~\cite{openai2025deepresearch}, and OpenClaw~\cite{openclaw} that tackle intricate tasks via long-horizon tool interactions.
As agentic RL integrates into LLM development pipelines,
rollout efficiency becomes a critical bottleneck. The ability to
sustain high-throughput trajectory generation
governs system efficiency and dictates the final model capabilities under realistic resource constraints.

\subsection{System Characterization}
\label{sec:background:system-characterization}
A standard RL training step comprises three phases:
(1) \emph{rollout} (trajectory generation), (2) \emph{inference} (reward and reference computation),
and (3) \emph{training} (policy model update).
As quantified in prior works~\cite{qin2025seer, hu2025taming}, the rollout phase dominates the RL training pipeline,
constituting a bottleneck fundamentally rooted in a pronounced \emph{straggler effect}.

\begin{figure}[t!]
    \centering
    \includegraphics[width=0.75\linewidth]{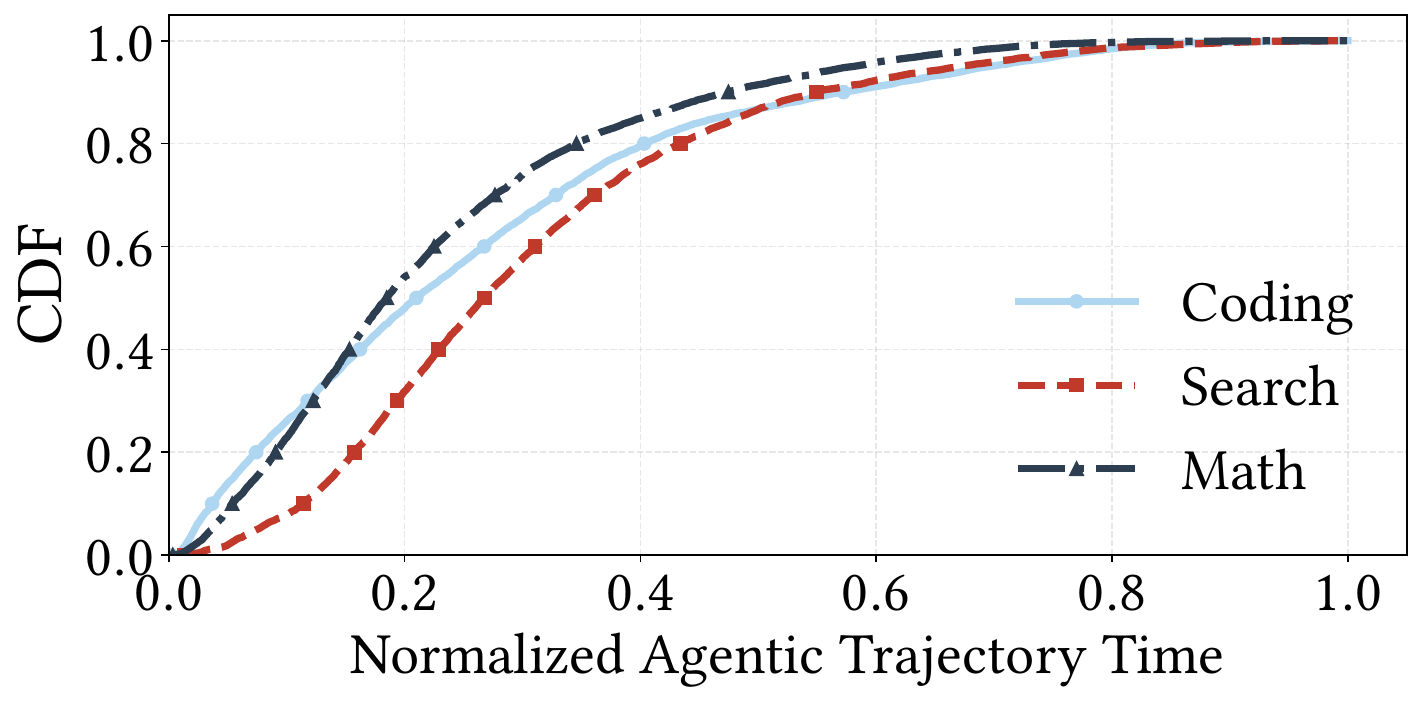}
    \vspace{-0.12in}
    \caption{CDF of normalized agentic trajectory completion time.}
    \vspace{-0.22in}
    \label{fig:background:time_cdf}
\end{figure}

This straggler effect stems from the inherent long-tailed distributions of agentic trajectories.
As Figure~\ref{fig:background:distribution} illustrates for a representative coding agent task,
both token generation and tool execution times are highly skewed. While most trajectories are
computationally light and terminate quickly, a sparse subset demands extensive multi-step
reasoning and prolonged tool interactions. In synchronous frameworks, these outliers become
dominant stragglers, forcing cluster-wide idleness that disproportionately inflates makespan
and degrades RL training throughput.

Ultimately, the rollout makespan is dictated by this critical, longest trajectory.
To quantify this, Figure~\ref{fig:background:time_cdf} profiles trajectory completion times during agentic RL rollout
using Verl~\cite{verl} and SGLang~\cite{zheng2024sglang}, normalized against the maximum time to control for task variance.
This profiling reveals a severe latency tail, where the maximum completion time exceeds
the median by over $4\times$. Because token counts ($N_{\text{tokens}}$) are algorithmically invariant
and tool execution ($T_{\text{tool}}$) is offloaded to elastic serverless infrastructure~\cite{amazon2014lambda,aliyun_fc,zhang2024jolteon,jin2023ditto},
our formulation isolates three GPU-centric targets per Formula~\ref{eq:makespan}:
queueing delay ($T_{\text{queue}}$), contention-induced interference overhead
($N_{\text{tokens}}\cdot(\alpha-1)\cdot T$), and base per-token time
($N_{\text{tokens}}\cdot T$).
Base per-token time is the contention-free time of average token
generation at batch size one.

\begin{figure}[t!]
    \centering
    \includegraphics[width=0.95\linewidth]{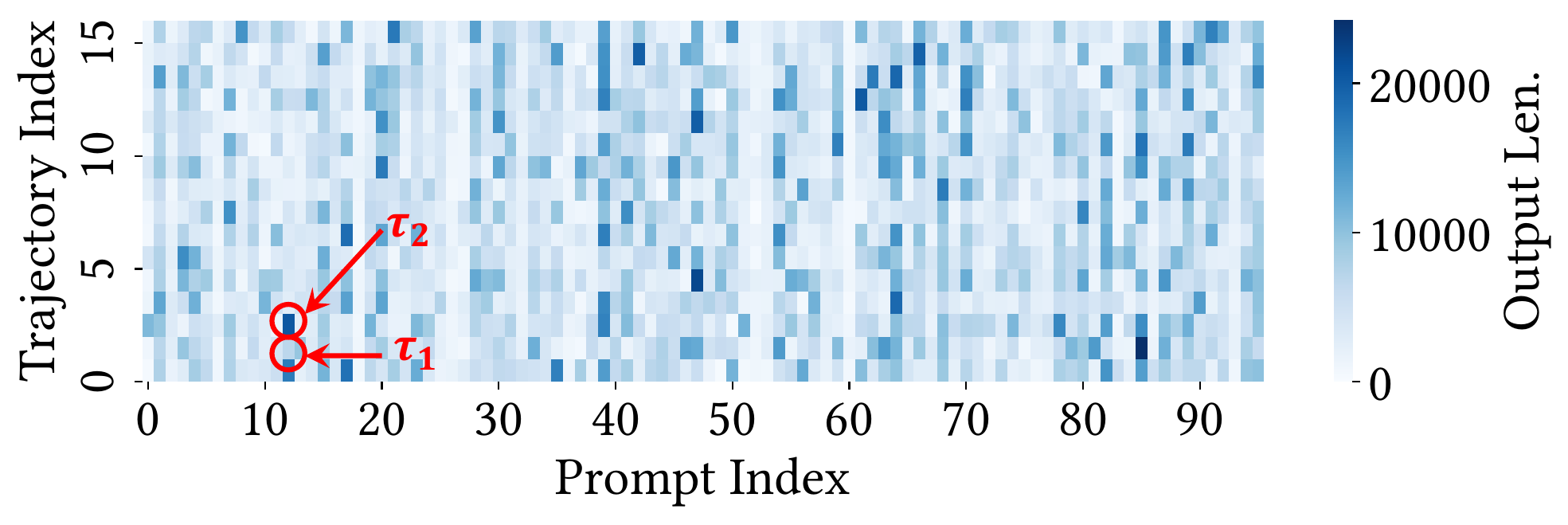}
    \vspace{-0.18in}
    \caption{Trajectory length distribution across different prompts.}
    \vspace{-0.22in}
    \label{fig:background:heatmap}
\end{figure}

\subsection{Challenges}
\label{sec:background:challenges}
Based on the system characterization, we identify three key challenges in agentic RL rollout optimization:
scheduling problem to minimize queueing delay, placement problem to minimize interference overhead, and
resource allocation problem to minimize base per-token time.

\parabf{Scheduling.}
As shown in Figure~\ref{fig:background:framework}, existing frameworks~\cite{verl,slime_github}
decouple LLM generation from tool execution, treating steps as stateless requests.
Consequently, scheduling defaults to round-robin~\cite{harchol2013performance} policy that
ignores the agentic trajectories, forcing long-tailed trajectories to accumulate excessive
queueing delays across multiple steps. While duration-based priority scheduling is a potential remedy,
the high stochasticity of agentic rollouts renders static prediction ineffective.
As illustrated in Figure~\ref{fig:background:heatmap}, identical prompts often
yield highly divergent trajectory lengths due to dynamic environment feedback.
For instance, two trajectories ($\tau_1$ and $\tau_2$) sharing a prompt might generate similar initial code;
however, if $\tau_2$ fails the example test, it triggers multiple rectification
steps, drastically extending its trajectory.
This unpredictability makes minimizing queueing delay a fundamental challenge.

\parabf{Placement.}
Existing frameworks~\cite{verl,slime_github} rely on cache-affinity or least-load placement,
neither of which mitigates the straggler effect. \emph{Cache-affinity} statically binds trajectories
to specific rollout workers to maximize prefix cache hits.
However, unknown trajectory durations inevitably trigger severe load imbalance
where some workers idle prematurely while others stall on long-tailed trajectories.
Conversely, \emph{least-load} balancing redistributes trajectories per-step,
incurring prohibitive cache recomputation and exacerbating the interference coefficient ($\alpha$)
of long-tailed trajectories.
As shown in Figure~\ref{fig:background:interference},
equalizing worker loads inadvertently co-locates long-tailed trajectories
with numerous short ones, forcing long-tailed trajectories to execute at high batch sizes
that inflate per-token time through increased memory and computation contention.
Ultimately, existing placement methods fail to resolve the straggler effect because they lack a global,
trajectory-aware perspective of the agentic RL rollout.

\begin{figure}[t!]
    \centering
    \includegraphics[width=0.65\linewidth]{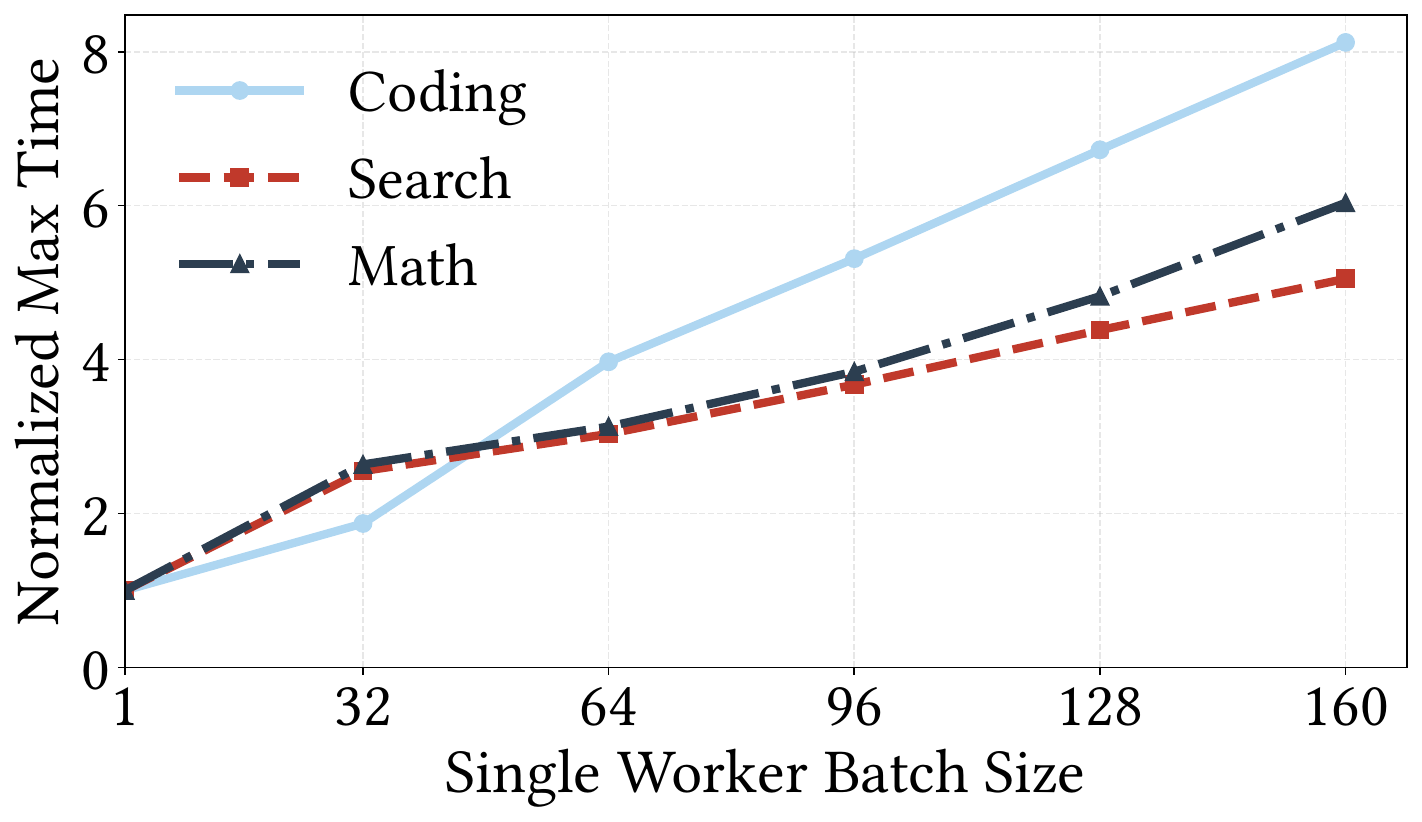}
    \vspace{-0.2in}
    \caption{Interference of long-tailed trajectories.}
    \vspace{-0.2in}
    \label{fig:background:interference}
\end{figure}

\parabf{Resource Management.}
Existing frameworks~\cite{verl,slime_github} enforce rigid, homogeneous GPU configurations across all rollout workers,
ignoring the inherent heterogeneity of agentic trajectories. Specifically, the large volume of short trajectories is
throughput-bound and benefits from lower model parallelism (MP) to minimize relative communication overhead.
Conversely, long-tailed trajectories are latency-bound and require higher MP to reduce per-token time.
As shown in Figure~\ref{fig:background:tp_sizes} (where $4\times2$ denotes four workers with two GPUs each),
there is a latency-throughput trade-off: scaling data parallelism maximizes throughput but sacrifices model parallelism,
severely inflating per-token time.
Consequently, as shown in Figure~\ref{fig:design:resource-manager:insight}(a),
each worker processes one agentic trajectory.
Under homogeneous regimes,
under-provisioned long-tail trajectories incur elevated per-token time.
This mismatch necessitates trajectory-adaptive resource management that dynamically tunes MP degrees to align GPU resources with specific workload demands.

\section{\sysname Overview}
\label{sec:overview}

We propose \sysname, a distributed system that mitigates agentic RL stragglers via a decoupled design
(Figure~\ref{fig:overview:overview}). It separates global orchestration from local execution: a
\textbf{control plane} dictates the \emph{when}, \emph{where}, and \emph{how}
of trajectory execution, while a \textbf{data plane} handles the underlying runtime.

\parabf{Control Plane.}
The control plane serves as the centralized brain of \sysname, maintaining a global view of
cluster resources and agentic trajectory states. It is composed of three synergistic modules that collectively
optimize the rollout.

\parait{\underline{Trajectory-level Scheduler (\emph{when}).}}
The scheduler uses a trainable runtime predictor
that fuses static prompt analysis with dynamic runtime context to estimate trajectory length.
Since prediction fidelity improves monotonically as trajectory context accumulates,
the scheduler adopts \emph{progressive priority scheduling}.
By iteratively refining length estimates, it escalates the priority of
long-tailed trajectories, granting them execution precedence to minimize
cumulative queueing delay ($T_{\text{queue}}$).

\begin{figure}[t!]
    \centering
    \includegraphics[width=0.73\linewidth]{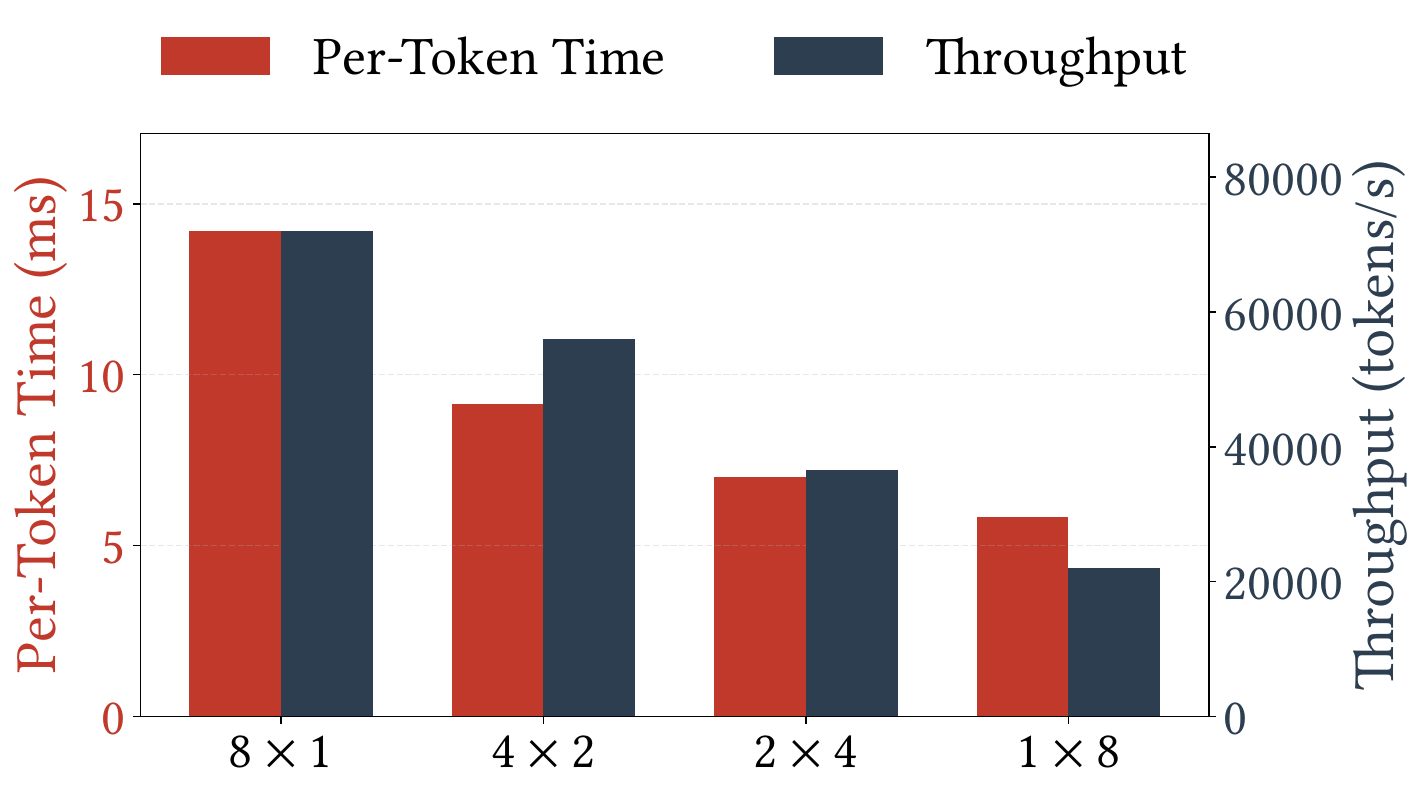}
    \vspace{-0.2in}
    \caption{Performance under different resource allocation strategies.}
    \vspace{-0.26in}
    \label{fig:background:tp_sizes}
\end{figure}

\parait{\underline{Trajectory-aware Placement (\emph{where}).}}
This component maps agentic trajectories to rollout workers with a two-phase strategy.
Initially, \emph{presorted dynamic programming}
spatially segregates long-tailed trajectories from short ones,
minimizing the interference coefficient ($\alpha$) of long-tailed trajectories.
At runtime, the engine monitors for load divergence
caused by evolving prediction accuracy (e.g., initially misclassified long-tailed trajectories).
Upon detection, it triggers \emph{opportunistic trajectory migration},
instructing the data plane to migrate trajectories during non-blocking tool
execution intervals.

\parait{\underline{Resource Manager (\emph{how}).}}
This component replaces homogeneous provisioning with a
\emph{trajectory-adaptive resource allocation} plan tailored to the
specific parallelism requirements of each agentic trajectory. It assigns high-degree model parallelism (MP)
to latency-sensitive long-tailed trajectories to accelerate per-token time ($T$),
while utilizing lower MP degrees for throughput-oriented short trajectories.

\parait{\underline{Tool Manager.}}
The tool manager orchestrates tool invocations via an elastic serverless backend,
eliminating cluster management overhead. By leveraging FaaS optimizations~\cite{mahgoub2022orion,zhang2024jolteon},
it effectively mitigates cold-start latencies and absorbs transient computational bursts during agentic rollout.
The pay-as-you-go billing model reduces the total cost of ownership compared to over-provisioned static resources.

\parabf{Data Plane.}
The data plane comprises a cluster of adaptive rollout workers that execute
intensive agentic workloads. Upon receiving control plane
directives, these workers materialize heterogeneous workers to orchestrate the interleaved
execution of model reasoning and tool execution.

\parait{\underline{Opportunistic State Migration.}}
The data plane masks migration overhead by exploiting the natural
boundaries between model reasoning and tool execution.
When a trajectory triggers a tool call, the worker yields its GPU resources.
\sysname leverages this idle interval to execute \emph{trajectory migration},
transferring the KV cache to a target worker designated by the control plane
without pausing the critical execution path.

\begin{figure}[t!]
    \centering
    \includegraphics[width=\linewidth]{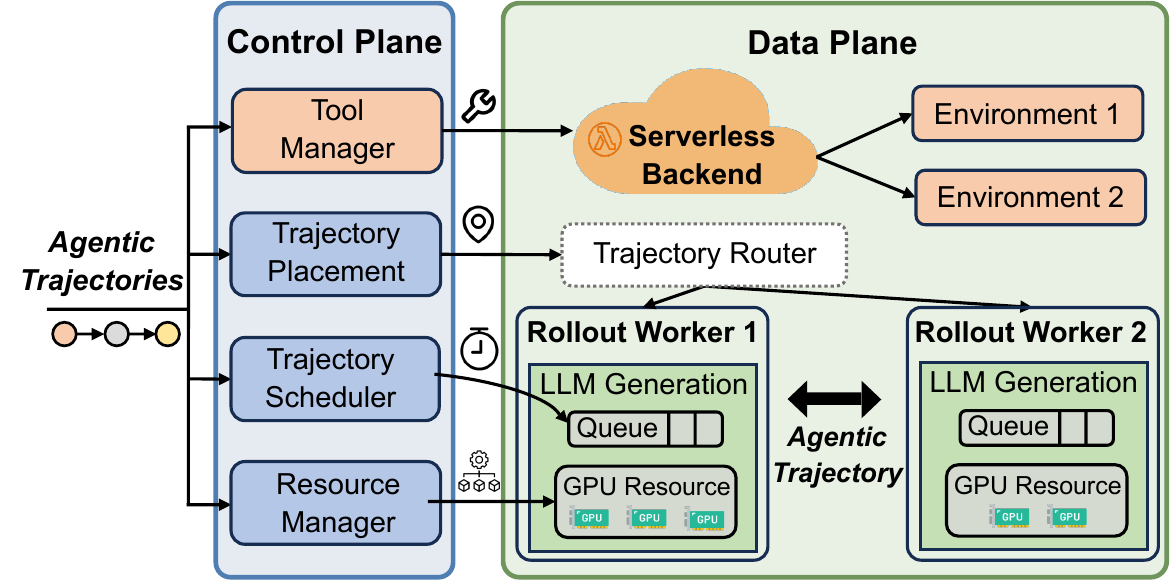}
    \vspace{-0.24in}
    \caption{Overview of \sysname.}
    \vspace{-0.2in}
    \label{fig:overview:overview}
\end{figure}

\parait{\underline{Runtime Telemetry.}}
The data plane closes the feedback loop by streaming real-time execution metrics, such
as current trajectory context, tool execution latency, and cache usage, back
to the control plane. This telemetry allows the runtime trajectory predictor to
correct its estimates and enables the scheduler to refine priorities for subsequent steps.

\section{Trajectory-level Scheduler}
\label{sec:design:scheduler}

To mitigate the queuing delay of long-tail trajectories in Formula~\ref{eq:makespan}, we propose \emph{trajectory-level scheduling}, shifting agentic
rollout from a step-centric to a trajectory-centric paradigm. As this approach relies on
trajectory identification, we first introduce \emph{progressive trajectory prediction} to dynamically identify potential
long-tail trajectories at runtime.

\subsection{Progressive trajectory Prediction}
\label{sec:design:scheduler:progressive-trajectory-prediction}

\parabf{Problem.}
Conventional prompt analysis~\cite{zhong2025streamrl,qin2025seer,he2025history} relies on static mechanisms that
estimate agentic trajectory length \emph{a priori} using historical data.
However, these methods fail to capture the dynamic, multi-step nature of agentic rollouts.
In RL algorithms like PPO~\cite{schulman2017proximal} and GRPO~\cite{shao2024deepseekmath}, a single prompt spawns a
group of agentic trajectory samples for advantage estimation.
To encourage exploration, high sampling temperatures are often employed,
inherently inducing high output variance.
Furthermore, trajectory length is also dictated by dynamic environmental feedback.
For instance, in a coding agent, identical prompts can yield divergent trajectories:
$\tau_1$ may pass example test cases in one step, while $\tau_2$ requires multiple
rectification steps, significantly extending $\tau_2$'s lifespan.
Thus, as Figure~\ref{fig:background:heatmap} shows,
trajectories within the same group exhibit significant length divergence
(i.e., intra-group variance). This runtime stochasticity renders static prediction ineffective for precise agentic trajectory prediction.

\parabf{Methodology.}
We propose \emph{progressive trajectory prediction} to capitalize on the iterative nature of agentic interaction.
As shown in Figure~\ref{fig:design:scheduler:progressive-trajectory-prediction},
the predictor monotonically refines length estimates as LLM generations and tool outputs enrich the agentic trajectory's context.
Notably, the initial step's execution plan serves as a strong semantic indicator,
anchoring the prediction with an accurate early-stage estimate.
To train the predictor, we harvest historical trajectories and
decompose them into \texttt{(context, remaining\_length)} tuples.
We leverage these data to fine-tune a lightweight pre-trained regression model (i.e., Qwen-0.6B).
At runtime, \sysname invokes this model after each step to update estimates and progressively reduce uncertainty.

We address potential overhead concerns through efficient model design.
First, training cost is trivial. The lightweight regression model requires only minutes to converge.
Second, deployed as a microservice, the model incurs negligible inference latency due to its compact
parameter size. Crucially, as shown in Figure~\ref{fig:design:scheduler:progressive-trajectory-prediction},
the prediction is performed \emph{asynchronously} alongside tool execution. This parallelism effectively
masks the inference cost, ensuring zero additional overhead on the agentic trajectory's critical path.

\begin{figure}[t!]
    \centering
    \includegraphics[width=0.95\linewidth]{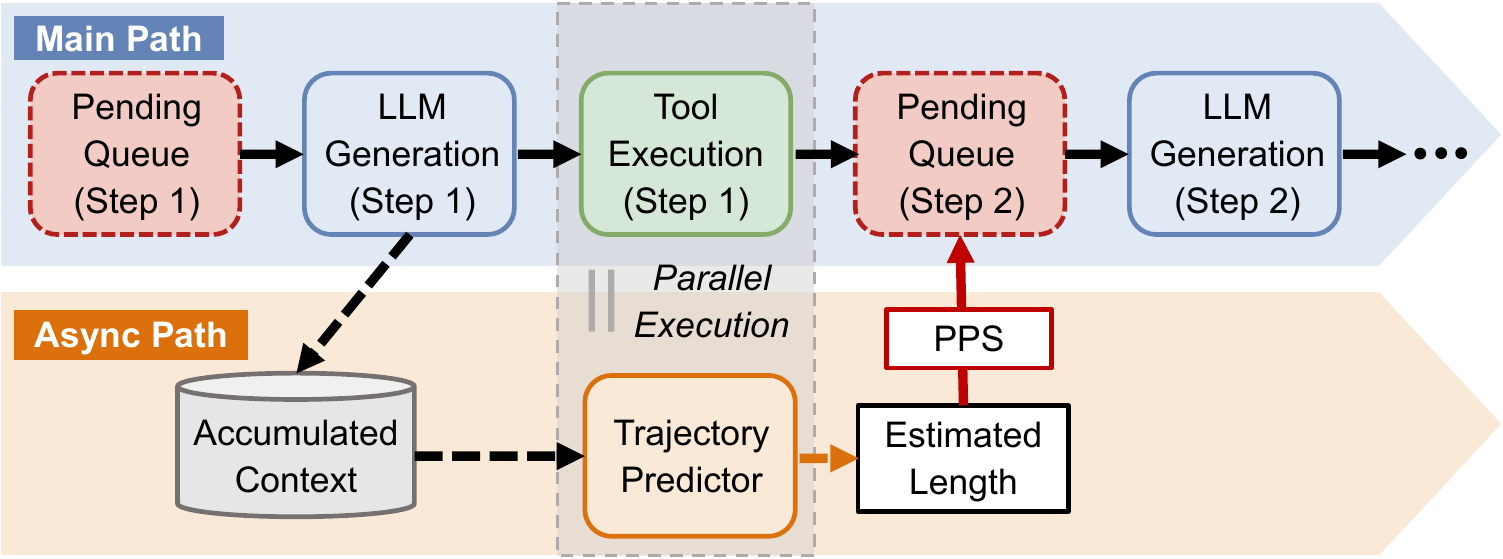}
    \vspace{-0.1in}
    \caption{Progressive trajectory prediction.}
    \vspace{-0.2in}
    \label{fig:design:scheduler:progressive-trajectory-prediction}
\end{figure}

\subsection{Progressive Priority Scheduling}
\label{sec:design:scheduler:progressive-priority-scheduling}
\paraf{Problem.}
As shown in Figure~\ref{fig:background:framework}, existing agentic RL frameworks~\cite{verl,slime_github} are architected for stateless interactions, typically decoupling
LLM generation from tool execution. In this paradigm, the system views an agentic trajectory not as a continuous
lifecycle, but as a fragmented sequence of independent requests. Consequently, scheduling defaults
to a trajectory-agnostic round-robin policy, where the execution quantum is rigidly limited to a
single step regardless of the trajectory's accumulated progress.
Specifically, every time a trajectory returns from a tool execution, it is treated as a de novo LLM generation request and
relegated to the tail of the waiting queue.
For a long-tail trajectory requiring $M$ steps, this imposes a recurring queuing penalty across $M$ distinct scheduling rounds.
By disproportionately penalizing long-tail trajectories with compounding queuing delays,
this approach proves detrimental to the global rollout makespan.

\parabf{Methodology.}
To optimize rollout makespan, we use \emph{progressive priority scheduling} (PPS),
an adaptive approximation of the longest-processing-time-first (LPT) discipline~\cite{graham1969bounds}.
LPT minimizes batch makespan by strictly prioritizing long-duration tasks, thereby mitigating
the cluster-wide idleness caused by the trailing stragglers.
PPS operationalizes this principle by dynamically mapping the progressive trajectory prediction to scheduling priorities.
As execution unfolds and prediction fidelity improves, PPS reorders the LLM inference requests in the pending queue for each rollout worker.
This ensures that identified long-tail trajectories gain execution precedence over short trajectories.

\begin{algorithm}[t]
\caption{Progressive Priority Scheduling}
\label{alg:pps}
\small
\begin{algorithmic}[1]
\Require Pending queue $Q$, active set $A$, predictor $\mathcal{P}$
\Statex \textbf{// Invoked when trajectory $\tau$ returns from tool execution}
\Function{Schedule}{$\tau$}
    \State $\tau.\text{pred\_len} \gets \mathcal{P}(\tau.\text{context})$ \Comment{progressive prediction}
    \State $\tau.\text{priority} \gets \tau.\text{pred\_len}$ \Comment{longer $\Rightarrow$ higher priority}
    \State Insert $\tau$ into $Q$
    \State Sort $Q$ by descending priority
    \Statex \textbf{// Preemptive execution}
    \State $r_{\min} \gets \arg\min_{r \in A} r.\text{priority}$
    \If{$Q.\text{top}.\text{priority} > r_{\min}.\text{priority}$}
        \State Evict $r_{\min}$ from $A$; persist KV cache
        \State Move $r_{\min}$ to $Q$
        \State Promote $Q.\text{top}$ to $A$
    \EndIf
\EndFunction
\end{algorithmic}
\end{algorithm}

\parabf{Preemptive Execution.}
To rigorously enforce the LPT discipline, we integrate \emph{preemptive execution} into SGLang~\cite{zheng2024sglang}.
This mechanism extends prioritization beyond the waiting queue, enabling high-priority
pending generation requests to interrupt active low-priority requests.
Specifically, whenever a pending request outranks the lowest-priority active request,
the system preempts the active request with the lowest priority, relegating it to the waiting queue while persisting its
prefix cache. The high-priority request is immediately promoted to the newly vacated slot.
This guarantees immediate execution for high-priority requests, significantly minimizing the queuing latency of long-tail trajectories.
Algorithm~\ref{alg:pps} illustrates the detailed pseudocode.

\section{Trajectory-aware Placement}
\label{sec:design:placement}
To mitigate the interference factor of long-tailed trajectories in Formula~\ref{eq:makespan},
we propose \emph{trajectory-aware placement}.
We begin by formalizing the placement optimization problem.
Subsequently, we introduce our solution: \emph{optimal} presorted dynamic programming algorithm, augmented by runtime trajectory migration.

\subsection{Problem formulation}
\label{sec:design:placement:problem-formulation}
In agentic rollout, multiple LLM rollout workers are deployed to execute concurrent trajectories through parallel batching mechanisms.
However, this introduces inevitable interference for token computation due to GPU resource contention.
We assume that the average base per-token time (batch size $=1$) is a constant $T$ and let $L$ denote the trajectory length function.
Given $n$ agentic trajectories $\{\tau_1, \ldots, \tau_n\}$ and $m$ LLM rollout workers, we seek a partitioning strategy $\{g_1, \ldots, g_m\}$,
where $g_i$ contains a set of agentic trajectories $\{\tau_{i_1}, \ldots, \tau_{i_k}\}$ and
$g_i$ is assigned to the $i$-th worker. We define $F(g_i)$ as the interference factor for group $g_i$.
To minimize the global rollout makespan, we define the optimization objective as:
\begin{equation}
    \min_{\{g_1, \ldots, g_m\}} \max_{i=1}^{m} \left(F(g_i) \times \max_{j=1}^{k} L(\tau_{i_j}) \times T\right)
    \label{eq:placement:objective}
\end{equation}
The goal is to minimize the completion time across $m$ groups,
where the group execution time is the product of its interference factor and the duration of its longest trajectory.

This optimization problem is fundamentally NP-hard. In heterogeneous settings (i.e., resource configurations vary across workers),
the search space scales as $O(m^n)$; even in the homogeneous case, complexity is governed by
the Stirling number of the second kind, $S(n, m)$~\cite{stanley2011enumerative}. This combinatorial
explosion renders exact enumeration intractable. Furthermore, the interference function $F$
lacks a closed-form analytic expression, preventing the use of off-the-shelf solvers.
These barriers preclude exact resolution during the agentic RL rollout,
necessitating our specialized \emph{presorted dynamic programming} algorithm.

\subsection{Presorted Dynamic Programming}
\label{sec:design:placement:presorted-dp}
We ground our placement algorithm on some simplifying premises. First, we treat trajectory lengths as
known \emph{a priori}. Since the initial prediction remains subject to estimation errors (\S~\ref{sec:design:scheduler:progressive-trajectory-prediction}),
we address the prediction deviations via trajectory migration (\S~\ref{sec:design:placement:trajectory-migration}).
Second, we model the cluster as homogeneous workers with the same resource configurations.
We subsequently relax this constraint in \S~\ref{sec:design:resource-manager},
extending our solution to heterogeneous worker configurations.
Finally, we posit that the interference factor is a monotonically increasing function determined exclusively by the size of the agentic trajectory group,
a property that holds empirically for standard workloads.

\parabf{Insight.}
We derive a critical structural insight:
\begin{lemma}\label{lemma1}
    Given agentic trajectories in descending length order,
    there exists an optimal partitioning strategy where each group constitutes a
    contiguous subsequence of the sorted list.
\end{lemma}
\noindent\textbf{\emph{Proof.}} Consider an optimal partition with two (generalizable to $m$)
groups $g_1$ and $g_2$, where $\tau_1 \in g_1$ and $L(\tau_1) \ge \dots \ge L(\tau_n)$.
Suppose the partition is non-contiguous. Then there exists $\tau_i \in g_2$ such that $L(\tau_i) > L(\tau_{min})$
for some $\tau_{min} \in g_1$. We swap $\tau_i$ and $\tau_{min}$. Because group sizes are invariant,
interference factors $F(g_1)$ and $F(g_2)$ remain unchanged according to the aforementioned premise.
In $g_1$, the makespan is still dictated by $\tau_1$ since $L(\tau_1) \ge L(\tau_i)$, so $g_1$'s completion time
is unchanged. In $g_2$, replacing $\tau_i$ with the shorter $\tau_{min}$ ensures the maximum trajectory length is non-increasing.
Consequently, the global rollout makespan is non-increasing.
Iterative swaps eventually yield a contiguous partition $\{g_1^*, g_2^*\}$ satisfying Lemma~\ref{lemma1}
without increasing the rollout makespan of the optimal original partition.

Leveraging this insight, we presort agentic trajectories by descending length and
enforce a contiguity constraint on partitioning, which guarantees
that the optimal solution derived under this constraint is globally optimal.

\parabf{Dynamic Programming.}
Given the presorted order, our problem becomes analogous to the linear partition problem~\cite{skiena2020algorithm}.
The contiguous partitioning constraint reduces the search space from the Stirling number of the
second kind, $S(n, m)$, to $\binom{n-1}{m-1}$. While significantly smaller, this magnitude remains
intractable for runtime enumeration. To address this, we propose a dynamic programming algorithm
that efficiently resolves the optimal partition solution.

\parait{\underline{State Initialization.}}
We define $dp[i][j]$ as the optimal makespan when partitioning the first $i$ trajectories across $j$ workers.
The initialization focuses on the single worker case ($j=1$).
For $i \in [1, n]$, we set $dp[i][1] = L(\tau_1) \times T \times F(\{\tau_1, \ldots, \tau_i\})$, with the boundary condition $dp[0][0] = 0$.
In this formulation, $L(\tau_1)$ represents the maximum trajectory length, $F$ captures the interference factor,
and $T$ denotes the base per-token time.

\parait{\underline{State Transition.}}
The state transition function is defined as:
\begin{equation}
    dp[i][j] = \min_{k=1}^{i-1} \max \left\{
    \begin{aligned}
        & dp[k][j-1], \\
        & L(\tau_{k+1}) \times T \times F(\{\tau_{k+1}, \ldots, \tau_i\})
    \end{aligned}
    \right\}
    \label{eq:placement:state-transition}
\end{equation}
Formula~\ref{eq:placement:state-transition} recursively identifies the optimal partition index $k$ separating
the $j$-th group from the preceding $j-1$ groups.
The variable $k$ iterates through feasible split positions (i.e., $k \in [1, i-1]$), assigning $\{\tau_1, \ldots, \tau_k\}$
to previous workers and $\{\tau_{k+1}, \ldots, \tau_i\}$ to the current $j$-th worker.
The inner $\max$ operator models the parallel execution nature, i.e., the global makespan is dictated by the maximum worker's makespan.
Crucially, due to descending order sorting, $L(\tau_{k+1})$ represents the dominant length in the current group.
Finally, the outer $\min$ operator selects the $k$ that minimizes the global makespan, ensuring load balancing across workers.

\begin{figure}[t!]
    \centering
    \includegraphics[width=1.03\linewidth]{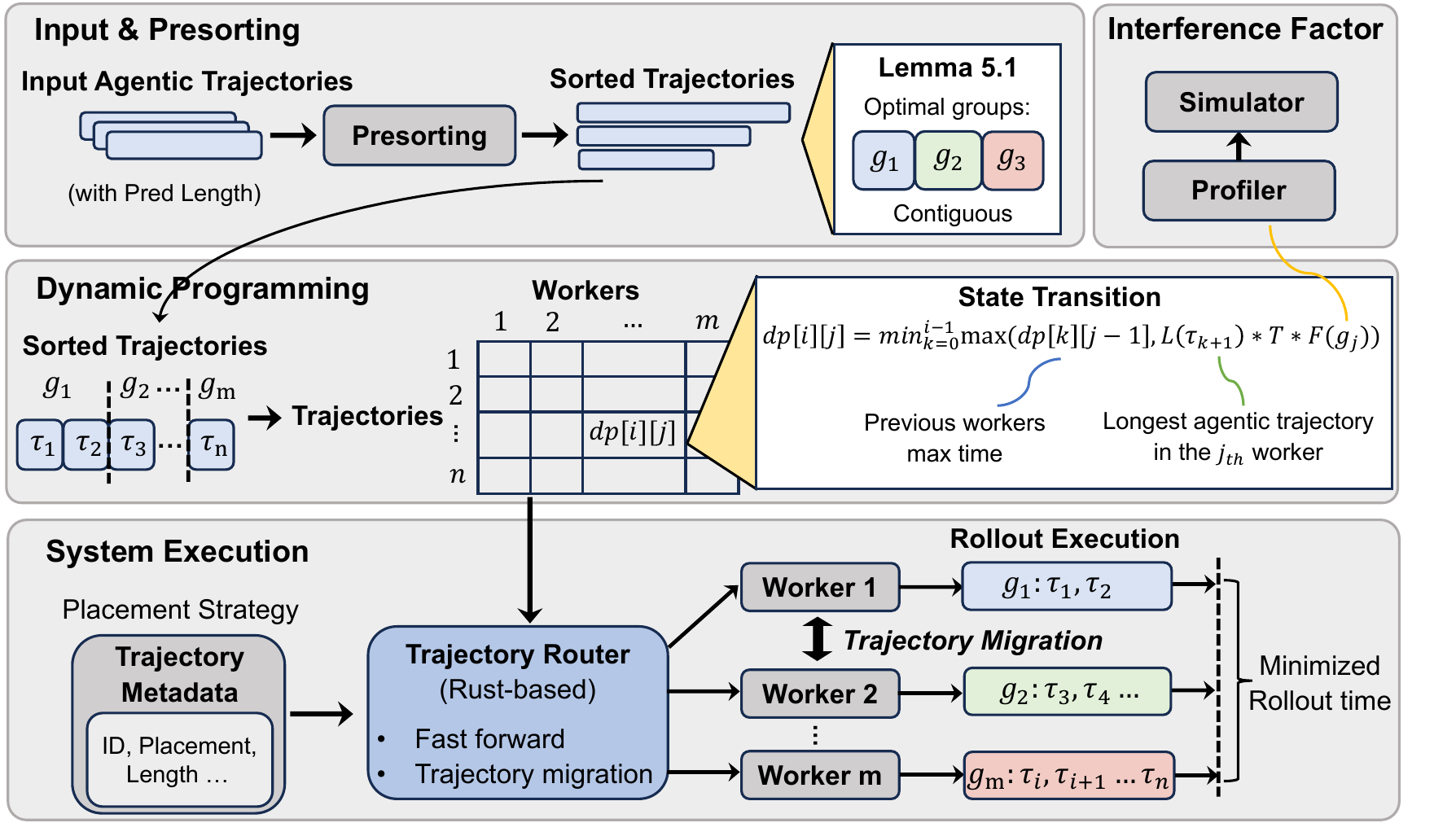}
    \vspace{-0.2in}
    \caption{Illustration of the placement algorithm.}
    \vspace{-0.2in}
    \label{fig:design:placement:illustration}
\end{figure}

The combination of presorting and dynamic programming yields a globally optimal solution based on Lemma~\ref{lemma1}.
Figure~\ref{fig:design:placement:illustration} illustrates the trajectory of this algorithm.
Its time complexity is $O(n^2 m)$, a drastic reduction from the combinatorial baseline $O(\binom{n-1}{m-1})$.
However, for large-scale workloads, $O(n^2 m)$ remains computationally substantial.
To further mitigate the overhead, we aggregate short trajectories below a defined threshold after sorting.
This heuristic reduces the effective input size $n$, significantly accelerating
execution with negligible impact on solution quality.

\parabf{Interference Factor.}
To derive the interference factor $F(g_i)$, we employ a profiler-based simulation approach.
We first construct a profiler to characterize per-token time across a spectrum of batch sizes.
These profiles feed into a simulator that models the concurrent execution of trajectories in $g_i$,
capturing batch-dependent interference and queuing delays.

\parabf{Agentic Trajectory Router.}
We implement a lightweight, Rust-based router to dispatch LLM generation requests.
This component maintains essential trajectory metadata, including placement assignments, predicted lengths,
and presorted ranks.
During initialization, the router synchronizes with the LLM rollout workers. Then, it ingests the partitioning
strategy from the control plane and routes each trajectory to its designated worker, strictly enforcing the
placement decisions derived from presorted dynamic programming.

\subsection{Trajectory Migration}
\label{sec:design:placement:trajectory-migration}
The presorted dynamic programming algorithm optimizes placement under the assumption
of perfect prediction accuracy. However, as noted in \S~\ref{sec:design:scheduler:progressive-trajectory-prediction},
prompt-only estimates suffer from inherent variance. Relying solely on this static assignment
exacerbates the resource contention for unanticipated long-tailed trajectories.
To mitigate this, we employ a runtime mechanism that adapts placement based on progressive prediction updates.

\parabf{Trajectory Migration Strategy.}
We employ trajectory migration to dynamically correct load imbalances. When a trajectory's predicted length
is updated, the trajectory router determines its new rank in the sorted order.
To avoid re-running the expensive dynamic programming algorithm, we scale the original partition sizes,
defined by each group's size $\{s_1, \ldots, s_m\}$ ($s_i =|g_i|$),
proportionally to the number of remaining active trajectories, $n^*$.
Specifically, the effective capacity of the $i$-th group becomes $s_i \times \frac{n^*}{n}$.
Using these scaled sizes, the router identifies the appropriate worker for the trajectory's new rank
and executes a migration if the target worker differs from the current host.

\parabf{KV Cache Migration.}
To mitigate the prefill overhead for long-tailed trajectories, we implement a live KV cache migration mechanism.
Upon reassignment, the router identifies the resident prefix cache and initiates a direct transfer to the target worker via GPU-Direct RDMA,
ensuring high-throughput and low-latency transmission. To manage concurrent migrations and prevent
endpoint contention, we implement a trajectory-aware transmission scheduler. The router prioritizes migration requests
in descending order of trajectory length. In each scheduling epoch, it greedily selects the longest available trajectory's migration request,
skipping any request that shares a source or destination worker with selected or running migration requests.
This transmission scheduler iteratively constructs batches of strictly parallel, non-conflicting migration requests.
By prioritizing long trajectories and enforcing endpoint exclusivity, this strategy maximizes bandwidth utilization 
while ensuring timely migration for critical long-tailed trajectories.

\section{Trajectory-adaptive Resource Manager}
\label{sec:design:resource-manager}
To reduce the base per-token time ($T$ in Formula~\ref{eq:makespan}) of long-tailed trajectories while preserving the high throughput of short
trajectories, we propose a trajectory-adaptive resource manager. We formally define the allocation problem and employ a
simulated annealing algorithm to efficiently converge to a near-optimal solution.

\subsection{Problem Formulation}
\label{sec:design:resource-manager:problem-formulation}
In \sysname's agentic rollout, we deploy multiple LLM workers with heterogeneous resource configurations
and model parallelism strategies. As illustrated in Figure~\ref{fig:design:resource-manager:insight}(b),
our core insight is to provision workers with higher degrees of model parallelism for long-tailed trajectories
(to minimize per-token time) while assigning lower degrees of model parallelism to short trajectories (to maximize throughput).
Formally, we allocate a total GPU budget $N = \sum_{i=1}^m N_i$ across $m$ workers with model parallelism (MP) degrees $\{N_1, \ldots, N_m\}$.
Transitioning to heterogeneous workers complicates the original placement problem formulation (\S~\ref{sec:design:placement:problem-formulation})
by requiring the joint optimization of trajectory partitions $\{g_i\}$ and MP allocations $\{N_i\}$.
To maintain tractability, we decouple it into two manageable subproblems, i.e., mapping and resource allocation.
solved via a two-stage heuristic: \emph{sort-initialized simulated annealing}.

\subsection{Sort-Initialized Simulated Annealing}
\label{sec:design:resource-manager:simulated-annealing}

\paraf{Mapping.}
The key insight for our mapping strategy is that the trajectory ordering established
in \S~\ref{sec:design:placement:presorted-dp} serves as a strong structural prior.
Specifically, the partitions $\{g_1, \ldots, g_m\}$ are inherently sorted in descending order
of their predicted lengths. Since our objective is to assign long-tailed trajectories to workers with higher model parallelism (to minimize per-token time),
we identically sort the worker resources $\{N_1, \ldots, N_m\}$
in descending order of their model parallelism degrees. We then deterministically assign the $i$-th
trajectory partition $g_i$ to the $i$-th worker possessing $N_i$ GPUs.

\begin{figure}[t!]
    \centering
    \includegraphics[width=0.86\linewidth]{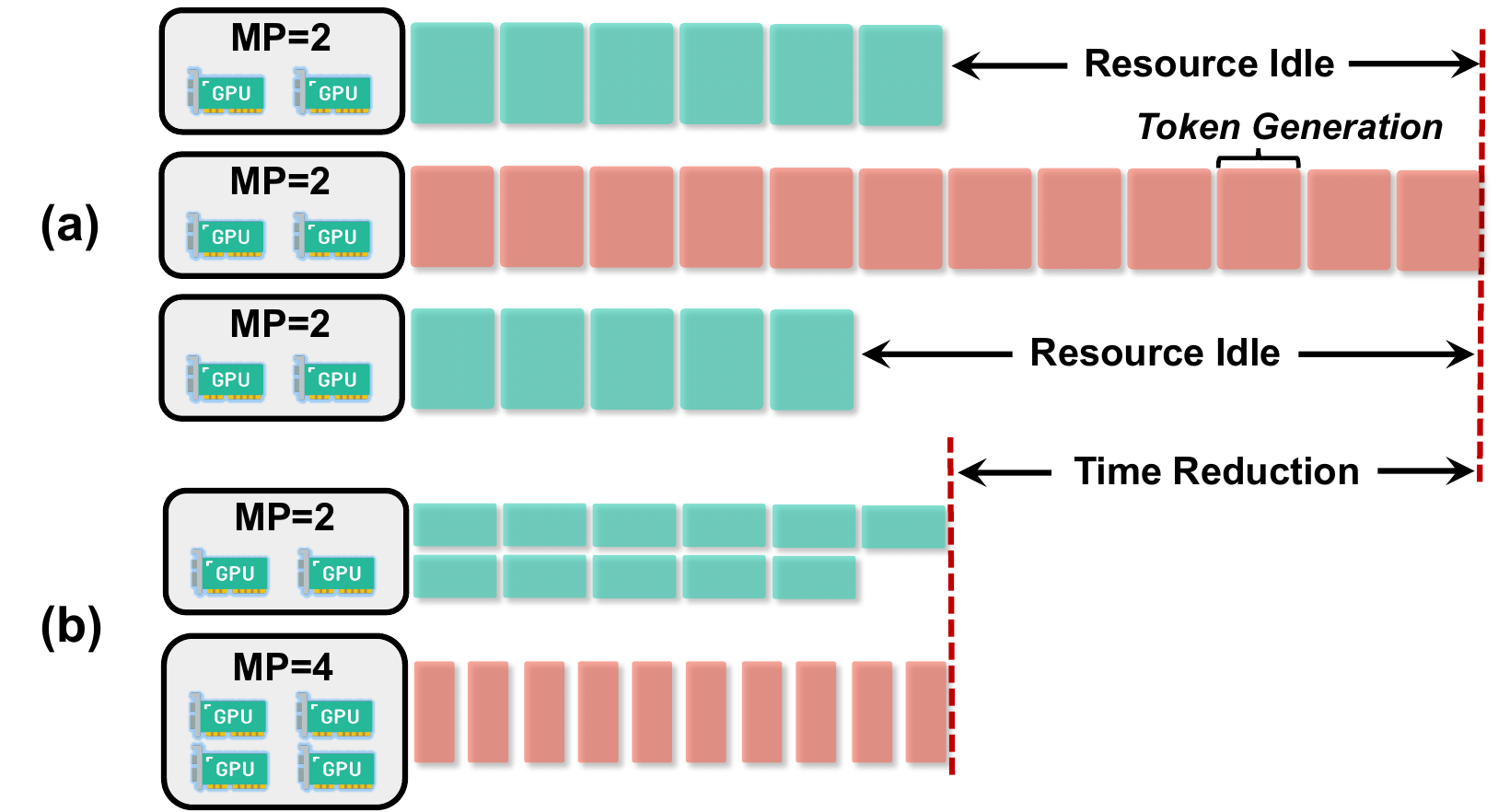}
    \vspace{-0.1in}
    \caption{Insight of heterogeneous resource allocation.}
    \vspace{-0.2in}
    \label{fig:design:resource-manager:insight}
\end{figure}

\parabf{Resource Allocation.}
Optimizing the GPU allocation $\{N_1, \ldots, N_m\}$ is critical, as each $N_i$ dictates
the per-token time for partition $g_i$ and the resulting global makespan (\S~\ref{sec:design:placement:problem-formulation}).
A naive baseline would exhaustively enumerate all valid allocations, evaluate the completion time
via the presorted dynamic programming (DP) algorithm for each, and select the global minimum. However, a single DP
execution requires $\mathcal{O}(m \cdot n^2)$ time, taking approximately $42$ milliseconds
for $n=6400$ and $m=16$ (\S~\ref{sec:evaluation:overhead}). Given the total GPU budget $N = \sum_{i=1}^m N_i$,
the number of valid allocations scales as $\binom{N-1}{m-1}$. Consequently,
an exhaustive search across this combinatorially explosive space is computationally intractable.

To efficiently navigate this space, we use a simulated annealing heuristic. We initialize the search
with a randomized, sorted allocation: each capacity $N_i$ is sampled from predefined model parallelism degrees,
and the resulting array is sorted in descending order.
We set the initial annealing temperature to the starting state's estimated completion time, computed by applying the
sort-initialized mapping and executing the presorted DP algorithm to determine trajectory placement.
During each search iteration, we perturb the current allocation by randomly applying one of
three state transitions: \emph{redistribution}, \emph{split}, or \emph{merge}.
While we unconditionally accept a proposed allocation that reduces overall completion time,
we also accept suboptimal states with a temperature-dependent probability to escape local optima.
The search terminates when the temperature drops below a predefined threshold $\epsilon$,
yielding the final configuration for the agentic rollout. Algorithm~\ref{alg:sa} details this pseudocode.

\begin{algorithm}[t]
\caption{Sort-Initialized Simulated Annealing}
\label{alg:sa}
\small
\begin{algorithmic}[1]
\Require Total GPU budget $N$, workers $m$, MP degrees $\mathcal{D}$, cooling rate $\alpha$, threshold $\epsilon$
\Ensure Optimal allocation $\{N_1^*, \ldots, N_m^*\}$
\State Sample $N_i \sim \mathcal{D}$ for $i \in [1, m]$ s.t. $\sum_i N_i = N$
\State Sort $\{N_1, \ldots, N_m\}$ in descending order \Comment{sort-initialized}
\State $C \gets \textsc{PresortedDP}(\{N_1, \ldots, N_m\})$ \Comment{initial makespan}
\State $T \gets C$; $C^* \gets C$; $\{N^*\} \gets \{N\}$
\While{$T > \epsilon$}
    \State $\{N'\} \gets \textsc{Perturb}(\{N\})$ \Comment{redistribute / split / merge}
    \State Sort $\{N_1', \ldots, N_m'\}$ in descending order
    \State $C' \gets \textsc{PresortedDP}(\{N_1', \ldots, N_m'\})$
    \State $\Delta \gets C' - C$
    \If{$\Delta < 0$ \textbf{or} $\text{rand}() < e^{-\Delta / T}$}
        \State $\{N\} \gets \{N'\}$; $C \gets C'$
    \EndIf
    \If{$C < C^*$}
        \State $C^* \gets C$; $\{N^*\} \gets \{N\}$
    \EndIf
    \State $T \gets \alpha \cdot T$
\EndWhile
\State \Return $\{N_1^*, \ldots, N_m^*\}$
\end{algorithmic}
\end{algorithm}

By coupling sort-initialized mapping with simulated annealing-based allocation,
we create a robust heterogeneous management heuristic.
While the former aligns long-tailed trajectories with high MP workers,
the latter navigates the combinatorial search space to optimize GPU provisioning.
Jointly, these mechanisms mitigate long-trajectory computation bottlenecks and preserve short-trajectory throughput,
ultimately minimizing the global rollout makespan.

\section{Evaluation}
\label{sec:evaluation}

We implement \sysname based on Verl~\cite{verl}, SGLang~\cite{zheng2024sglang}, and Ray~\cite{moritz2018ray}
with 15K lines of Rust, Python, and C++.
In our evaluation, we first evaluate the overall performance of \sysname
compared to state-of-the-art agentic RL systems, and then we conduct
ablation studies to evaluate the effectiveness of each component of \sysname.
In our ablation studies (\S~\ref{sec:evaluation:trajectory-level-scheduling}-\S~\ref{sec:evaluation:trajectory-adaptive-resource-manager}),
we evaluate each component individually while keeping all
other system components and experimental configurations identical to the overall performance evaluation.
Finally, we analyze the overhead of \sysname's key components in \S~\ref{sec:evaluation:overhead}.

\parabf{Testbed.}
We conduct our experiments on a cluster of eight servers,
totaling 64 GPUs. Each node is equipped with eight NVIDIA Hopper GPUs, 160 Intel Xeon CPU cores,
and 1.8 TB of host memory. Intra-node GPU communication utilizes 900 GB/s NVLink,
while inter-node networking is handled by 400 Gb/s NVIDIA Mellanox InfiniBand supporting GPUDirect RDMA.
Our system software stack includes PyTorch 2.8.0 and CUDA 12.2 (driver 535.161.08).
For distributed orchestration and messaging, we use Ray 2.49.0 and
ZMQ 4.3.5.

\begin{figure*}[t]
    \centering
    \includegraphics[width=1\linewidth]{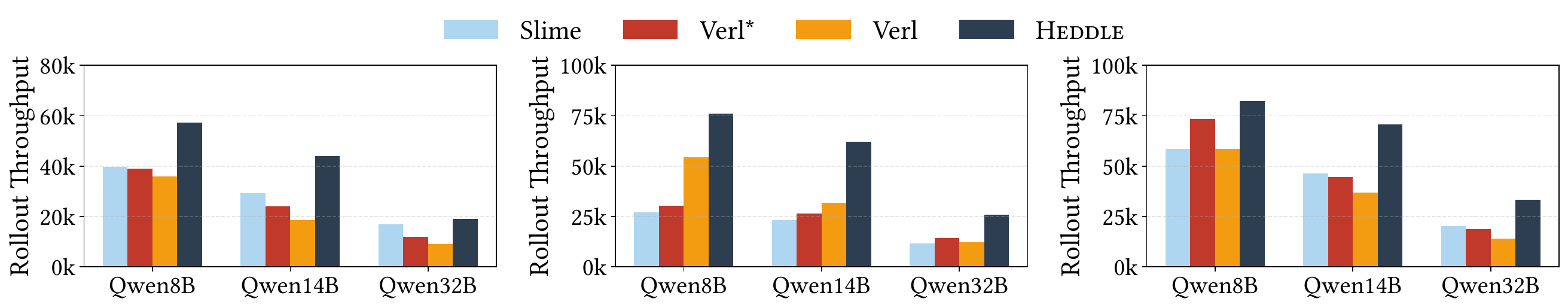}
    \footnotesize{\hspace*{0.04in}{(a) Coding agent.}\hspace*{\dimexpr\linewidth/4\relax}{(b) Search agent.}\hspace*{\dimexpr\linewidth/5\relax}{(c) Math agent.}}
    \vspace{-0.1in}
    \caption{Rollout throughput of agentic RL systems under different workloads and models.}
    \vspace{-0.05in}
    \label{fig:eval:overall}
\end{figure*}

\parabf{Models.}
We evaluate \sysname with instruction-tuned Qwen3~\cite{yang2025qwen3} models (8B, 14B, and 32B) possessing
foundational reasoning and tool-use capabilities. By benchmarking across these three parameter
scales, we evaluate \sysname's performance against baselines across a diverse spectrum of
computational intensities and resource budgets.

\parabf{Workloads.}
We evaluate \sysname across three representative agentic RL domains: coding, search, and math.
For the coding agent~\cite{yang2024swe}, we utilize the CodeForces dataset~\cite{li2022competition}, equipping the LLM
with a sandbox tool to execute code, run test suites, and check formatting.
The search agent~\cite{jin2025search} operates on the HotpotQA dataset~\cite{yang2018hotpotqa},
utilizing a web search tool~\cite{bing_search} for online information retrieval and multi-hop reasoning.
Lastly, the math agent~\cite{gou2023tora} leverages the DAPO-Math dataset~\cite{dapo_math_hf},
employing calculator and solver tools to resolve mathematical problems.
Across all agentic rollouts, we enforce a maximum output length of 40K tokens and
generate 16 samples per prompt using the GRPO~\cite{shao2024deepseekmath} RL algorithm. We configure LLM generation hyperparameters with a
temperature of $1.0$ and a top-$p$ of $0.9$, aligning with standard LLM post-training
practices~\cite{ouyang2022training,touvron2023llama}.

\parabf{Baselines.}
We compare \sysname against two state-of-the-art open-source agentic RL frameworks:
\begin{itemize}[leftmargin=*]
    \item \textbf{Slime}~\cite{slime_github} is an open-source RL training framework built on SGLang~\cite{zheng2024sglang}
    that employs a customized router to dispatch agentic trajectories. It achieves load balancing
    by dynamically routing individual LLM generation requests at each step to the least-loaded worker.
    \item \textbf{Verl}~\cite{verl} provides highly scalable agentic training abstractions
    alongside a hierarchical, hybrid programming model. To optimize prefix cache utilization,
    it utilizes a static, cache-aware placement strategy, pinning each complete agentic trajectory sample
    to a dedicated worker.
    \item \textbf{Verl*} extends Verl by integrating the SGLang-router~\cite{zheng2024sglang}.
    It employs a hybrid placement strategy: if the load skew (max/min)
    exceeds a threshold (e.g., 32), it uses least-load strategy;
    otherwise, it defaults to a cache-aware strategy.
\end{itemize}
Notably, both baselines inherently rely on round-robin scheduling and homogeneous resource allocation
during agentic rollout. Given that Slime is strictly coupled to SGLang,
we standardize our evaluation by configuring SGLang as the rollout backend for Verl as well to ensure a fair comparison.

\begin{figure*}[t]
    \centering
    \includegraphics[width=1\linewidth]{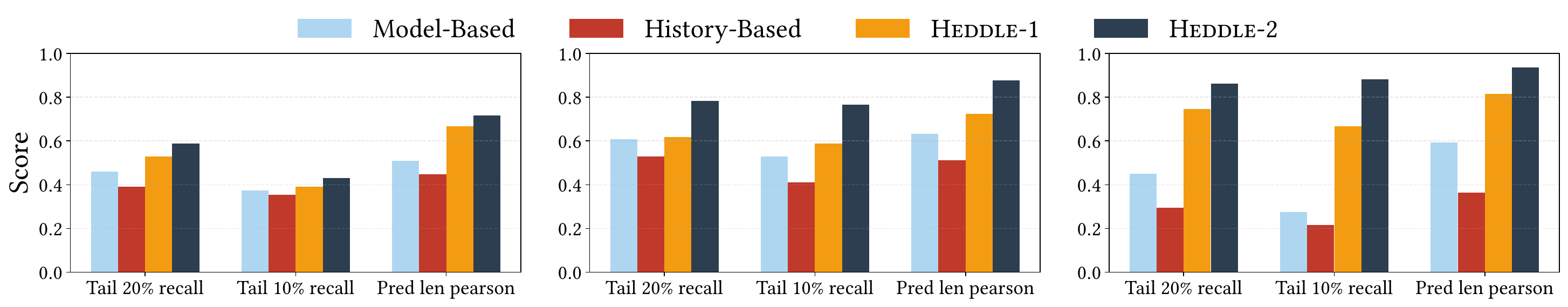}
    \footnotesize{\hspace*{0.04in}{(a) Coding agent.}\hspace*{\dimexpr\linewidth/4\relax}{(b) Search agent.}\hspace*{\dimexpr\linewidth/5\relax}{(c) Math agent.}}
    \vspace{-0.1in}
    \caption{Precision of progressive trajectory prediction under different workloads and models.}
    \vspace{-0.05in}
    \label{fig:eval:predictor}
\end{figure*}

\subsection{Overall Performance}
\label{sec:evaluation:overall}
We first benchmark \sysname against Slime and Verl. For the baseline systems,
which rely on homogeneous resource allocation, we configure the model parallelism degrees of
1, 1, and 2 for the Qwen3-8B, 14B, and 32B variants, respectively, to ensure maximum throughput.
We fix the baseline batch size at 100 per LLM rollout worker and apply the same global batch size
to \sysname. In contrast to the baselines, \sysname leverages the heterogeneous resource configurations
derived from our trajectory-adaptive manager (\S~\ref{sec:design:resource-manager}).
Given our focus on accelerating agentic RL rollouts, we use end-to-end rollout throughput (tokens/s)
as the performance metric.

Figure~\ref{fig:eval:overall} presents the end-to-end rollout throughput across various agentic RL tasks and model
sizes. \sysname outperforms both baselines, achieving speedups of $1.4\times$--$2.3\times$ over Verl, $1.1\times$--$2.4\times$ over Verl*, and $1.2\times$--$2.5\times$ over Slime.
Notably, these performance gains amplify as model size increases. This is because larger models inherently suffer from severe computation and memory
contention, which heightens the interference factor of the long-tail trajectories. This bottleneck significantly
degrades the baseline performance. \sysname, however, successfully mitigates this through its trajectory-aware placement. 
Comparing the baselines, Slime demonstrates superior efficiency on coding and math agent tasks by utilizing
a customized router to mitigate load imbalance across rollout workers. Conversely, Verl outperforms
Slime on the search agent due to the task's shorter sequences and frequent interaction steps,
which amplify prefill overhead and prioritize prefix cache locality.
In such cache-heavy scenarios, Verl's cache-aware placement proves more effective.
Verl* generally demonstrates intermediate performance between Verl and Slime,
representing a heuristic trade-off between cache affinity and load balance.
\sysname, however, implements a superior trajectory-aware placement strategy that simultaneously
maximizes prefix cache hit rates and minimizes interference from long-tailed trajectories,
outperforming both baselines across all task types.

\subsection{Effectiveness of trajectory-level Scheduling}
\label{sec:evaluation:trajectory-level-scheduling}
In this subsection, we evaluate the scheduler from two perspectives:
the precision of progressive trajectory prediction,
and the effectiveness of the progressive priority scheduling.

\parabf{Prediction.}
To assess the precision of the progressive trajectory prediction, we compare \sysname's predictor
against two prompt-based baselines: model-based and history-based prediction.
Model-based prediction~\cite{zhong2025streamrl} employs a lightweight deep learning model to infer prompt
complexity, while history-based prediction~\cite{he2025history,qin2025seer} estimates it with statistical heuristics from historical rollout data.
In contrast, \sysname leverages both the initial prompt and the runtime-generated context to
dynamically predict the complexity of an active agentic trajectory. For our evaluation metrics,
we report the recall of long-tailed trajectories (ranging from $0$ to $1$) to measure the precision
of relative trajectory length predictions, alongside the Pearson correlation coefficient
($-1$ to $1$) to quantify the relationship between predicted and actual trajectory lengths.

Figure~\ref{fig:eval:predictor} presents the prediction precision for the Qwen3-14B model across
three agentic tasks. Here, \sysname-1 and \sysname-2 denote the predictions following the first and second agentic steps,
respectively. Notably, \sysname consistently achieves a higher recall and Pearson correlation coefficient
than the baselines. Furthermore, \sysname-2 outperforms \sysname-1
since runtime-generated context enables progressively more accurate trajectory length predictions.

\parabf{Scheduling.}
To evaluate \sysname's progressive priority scheduling, we compare it against FCFS, Round-Robin (RR)
and Autellix~\cite{luo2025autellix}.
While existing agentic RL frameworks default to RR,
Autellix employs a shortest-job-first policy to prevent head-of-line blocking in online LLM agent serving.
We measure the rollout
time and the queueing delay of the longest-tailed trajectory using Qwen3-14B.
Because an agentic trajectory submits an LLM generation request at every agentic step,
we define the trajectory's queueing delay as the sum of the queueing delays incurred across all its steps.
As Figure~\ref{fig:eval:scheduler} shows, \sysname reduces end-to-end rollout
time by $1.1\times$--$1.26\times$ relative to the baselines.
This performance gain stems primarily from minimizing the queueing delays.
These results validate the effectiveness of our progressive priority scheduling.

\begin{figure}[t]
    \centering
    \includegraphics[width=1\linewidth]{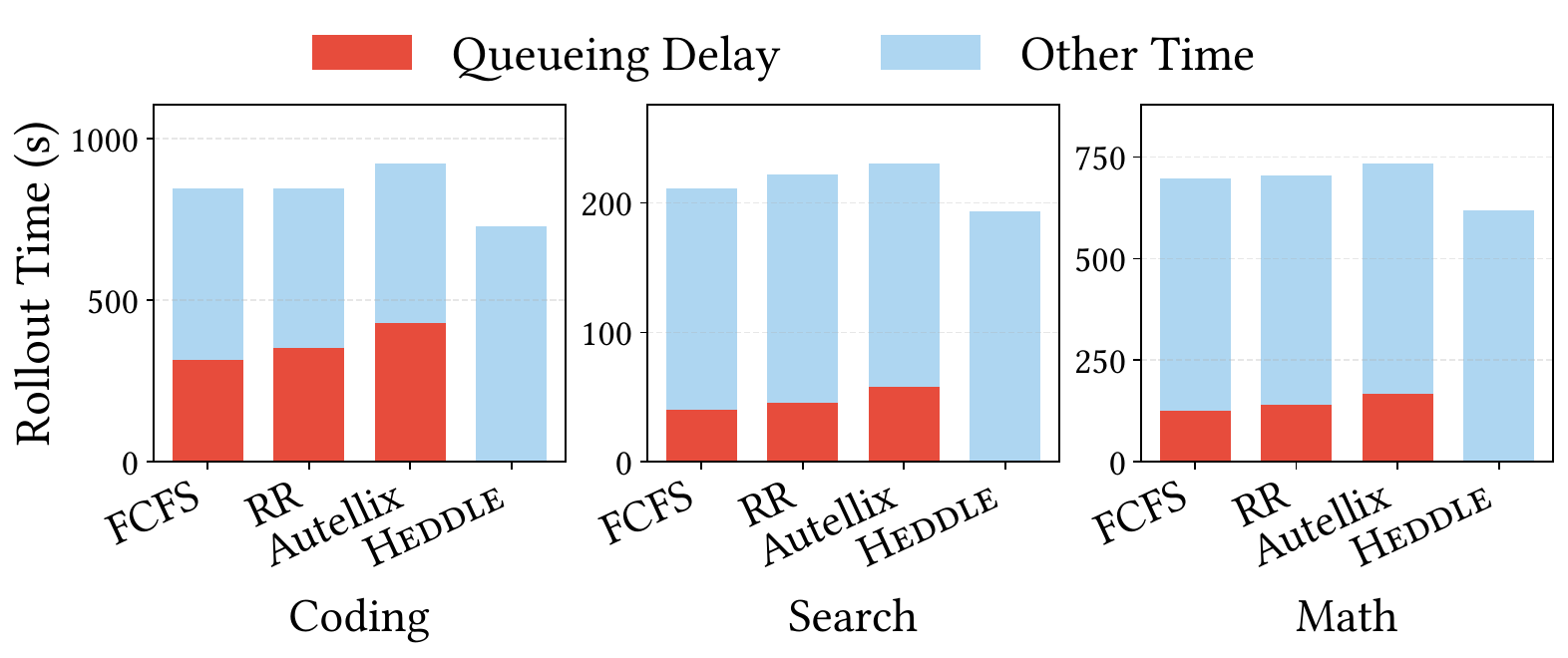}
    \vspace{-0.3in}
    \caption{Performance of trajectory-level scheduler.}
    \vspace{-0.15in}
    \label{fig:eval:scheduler}
\end{figure}

\subsection{Effectiveness of trajectory-aware Placement}
\label{sec:evaluation:trajectory-aware-placement}
To evaluate \sysname's trajectory-aware placement, we compare it against two baseline placement strategies:
\begin{itemize}[leftmargin=*]
    \item \textbf{Least-load} routes per-step LLM generation requests to the
    least-loaded worker if load imbalance exceeds a configurable threshold.
    Otherwise, it routes the request to the worker with the longest cached prefix.
    \item \textbf{Cache-aware} routes per-step LLM generation requests
    to the worker with the maximum prefix cache match, entirely disregarding potential load imbalances.
\end{itemize}
We evaluate the end-to-end rollout throughput using the Qwen3-14B model.
As Figure~\ref{fig:eval:dispatching} illustrates, \sysname achieves a $1.2\times$--$1.5\times$
higher throughput than both baselines.
While the cache-aware method maximizes prefix cache hits via static assignments,
the highly skewed output length of agentic trajectories induces severe load imbalance with this method.
The least-load method guarantees worker load balance while preserving
a degree of cache efficiency.
However, this method ignores the holistic agentic trajectory context, treating per-step requests in
isolation and incurring high interference for long-tailed trajectories.
\sysname overcomes these limitations by leveraging a trajectory-aware placement strategy, i.e.,
combining presorted dynamic programming with runtime trajectory migration,
to explicitly minimize the completion time of these critical long-tailed trajectories,
thereby reducing the overall rollout makespan.

\begin{figure}[t]
    \centering
    \includegraphics[width=0.78\linewidth]{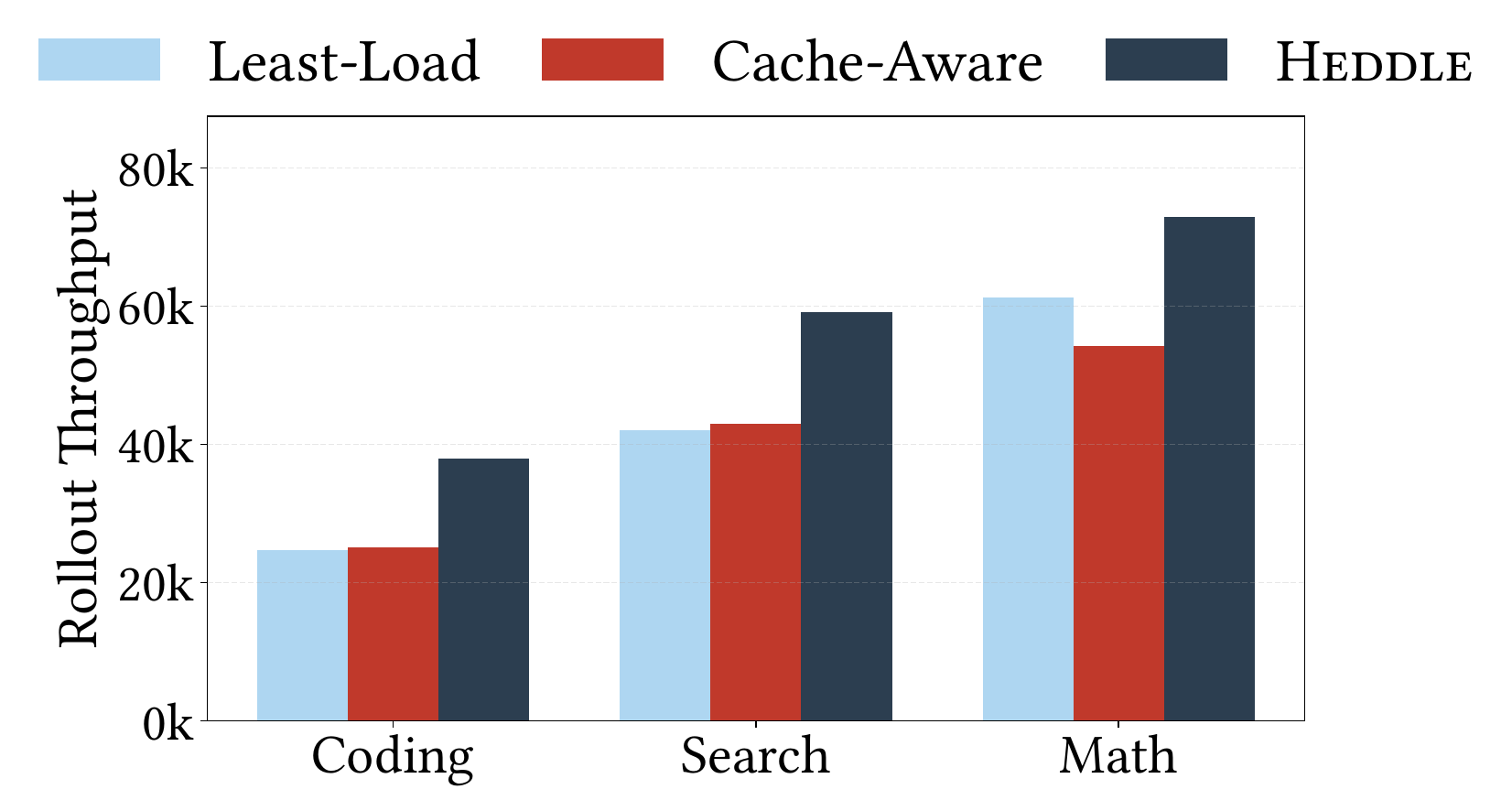}
    \vspace{-0.1in}
    \caption{Performance of trajectory-aware placement.}
    \vspace{-0.05in}
    \label{fig:eval:dispatching}
\end{figure}

\subsection{Effectiveness of Resource Manager}
\label{sec:evaluation:trajectory-adaptive-resource-manager}
To evaluate \sysname's trajectory-adaptive resource management,
we compare it against homogeneous resource allocation baselines,
where identical compute resources are assigned to each rollout worker.
We evaluate model parallelism sizes of $1$ (Fix-1) and $8$ (Fix-8),
representing throughput-optimized and latency-optimized configurations, respectively.
We measure the end-to-end rollout throughput with the Qwen3-14B model. Figure~\ref{fig:eval:resource}(a)
demonstrates that \sysname achieves a $1.1\times$--$1.3\times$ speedup over both static baselines.
Figure~\ref{fig:eval:resource}(b) tracks the number of active trajectories over time for the search agent.
While Fix-1 delivers peak initial throughput, its slow per-token time of long-tailed trajectories
ultimately bottlenecks the entire rollout. Conversely, Fix-8 minimizes per-token time for long-tailed trajectories
at the tail end of execution, but its persistently low overall throughput
throttles the rollout. \sysname resolves this \emph{performance trade-off} by dynamically
achieving low-latency for long-tailed trajectories while sustaining high throughput
for short trajectories, thereby outperforming both baselines.

\begin{figure}[t]
    \centering
    \includegraphics[width=1\linewidth]{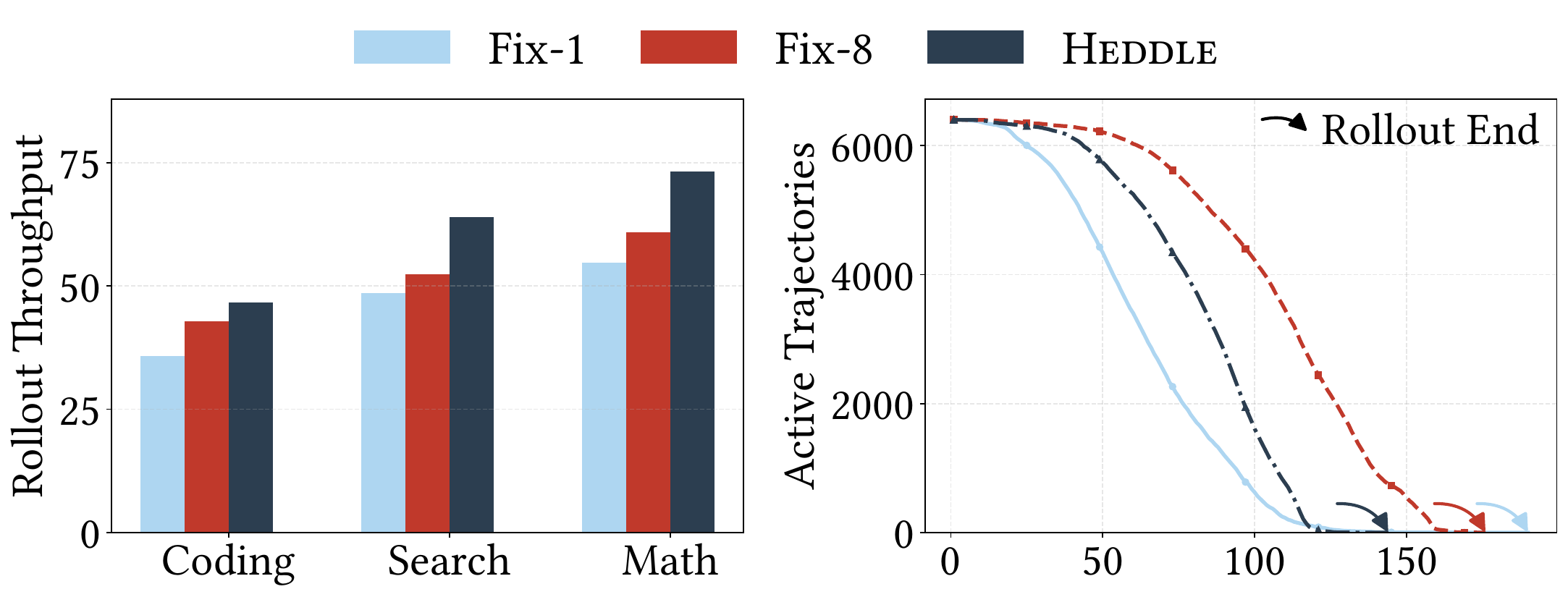}
    \footnotesize{\hspace*{0.2in}{(a) Rollout throughput.}\hspace*{\dimexpr\linewidth/8\relax}{(b) Active trajectories over time.}}
    \vspace{-0.15in}
    \caption{Performance of trajectory-adaptive resource management.}
    \vspace{-0.2in}
    \label{fig:eval:resource}
\end{figure}

\subsection{System Overhead}
\label{sec:evaluation:overhead}
\paraf{Data Plane.}
In this subsection, we first quantify the system overhead introduced by progressive trajectory
prediction (\S~\ref{sec:design:scheduler:progressive-trajectory-prediction})
and runtime trajectory migration (\S~\ref{sec:design:placement:trajectory-migration}).
Prediction overhead represents the latency of trajectory-length estimation,
while migration overhead accounts for the time required to transfer a trajectory's prefix cache between workers.
Notably, \sysname overlaps these two operations with standard tool execution.
Table~\ref{tab:eval:sys-overhead} details
the average tool execution times alongside the prediction and migration
overheads across various workloads.
As \sysname parallelizes these tasks, both overheads are effectively masked by tool execution
in most scenarios. Even when tool execution is exceptionally brief,
the exposed overhead remains negligible relative to the tens of seconds for per-step LLM generation.

\parabf{Control Plane.}
We next quantify the algorithmic overheads of placement (\S~\ref{sec:design:placement}) and resource management (\S~\ref{sec:design:resource-manager})
executed by the Ray controller worker. As Table~\ref{tab:eval:algorithm-overhead} shows,
placement overhead is negligible compared to the hundreds of seconds required for total rollout.
Furthermore, the resource management algorithm executes only periodically, amortizing its
cost across multiple training steps and rendering its system impact inconsequential.

\begin{table}[t!]
    \centering
    \resizebox{0.76\linewidth}{!} {
    \begin{tabular}{llccc}
        \toprule
        \textbf{Models} & & \textbf{Coding} & \textbf{Search} & \textbf{Math} \\
        \midrule
        \multirow{3}{*}{Qwen3-8B}
            & Tool Exec. (s)  & 0.46  & 1.42  & 0.051 \\
            & Pred. (s)       & 0.103 & 0.15  & 0.26  \\
            & Migration (s)  & 0.26  & 0.12  & 0.24  \\
        \midrule
        \multirow{3}{*}{Qwen3-14B}
            & Tool Exec. (s)  & 0.41  & 1.41  & 0.046 \\
            & Pred. (s)       & 0.099 & 0.15  & 0.27  \\
            & Migration (s)  & 0.27  & 0.15  & 0.25  \\
        \midrule
        \multirow{3}{*}{Qwen3-32B}
            & Tool Exec. (s)  & 0.45  & 1.43  & 0.054 \\
            & Pred. (s)       & 0.104 & 0.16  & 0.28  \\
            & Migration (s)  & 0.35  & 0.27  & 0.33  \\
        \bottomrule
    \end{tabular}
    }
    \vspace{0.1in}
    \caption{Prediction and migration overhead in \sysname.}
    \vspace{-0.4in}
    \label{tab:eval:sys-overhead}
\end{table}

\section{Discussion}
\label{sec:design:discussion}

\paraf{Asynchronous RL.}
Asynchronous RL~\cite{zhong2025streamrl} enhances RL throughput via partial rollout~\cite{team2025kimi} but requires staleness
thresholds to prevent gradient bias and preserve training convergence.
While this maximum staleness requirement maintains training stability, it does not
resolve the bottleneck of long-tailed agentic trajectories. \sysname can be seamlessly integrated
with staleness-bounded asynchronous RL, accelerating these stragglers without
introducing additional policy divergence or policy staleness.

\parabf{PD Disaggregation.}
Prefill-decode (PD) disaggregation~\cite{zhong2024distserve,patel2024splitwise} optimizes LLM serving by
applying heterogeneous model parallelism (MP) across prefill and decode stages.
However, it maintains homogeneous MP within each stage's worker pool.
\sysname is orthogonal to this approach and can be seamlessly integrated to provide intra-stage
heterogeneity. For instance, \sysname can dynamically assign higher
model parallelism to long-tailed trajectories within the prefill pool,
further accelerating agentic rollouts.

\parabf{Speculative Decoding.}
Speculative decoding (SD)~\cite{leviathan2023fast, hu2025taming} accelerates LLM decoding by using a small draft model
to generate candidate tokens, which are then verified in parallel by the target model.
However, SD degrades throughput at large batch sizes due
to the computational overhead of the rejected tokens.
SD is ill-suited for prefill-heavy agentic RL tasks where short generations across
frequent steps make tool-output prefill the dominant rollout bottleneck.
But \sysname is still effective by optimizing the prefill phase.
For decoding-heavy workloads, \sysname is orthogonal to SD and can be seamlessly integrated.
For instance, \sysname can dynamically route long-tailed trajectories to high model parallelism workers to reduce decoding time,
while concurrently applying SD to further minimize the decoding time.

\begin{table}[t!]
    \centering
    \resizebox{0.82\linewidth}{!} {
    \begin{tabular}{llccc}
        \toprule
        \textbf{Models} & & \textbf{Coding} & \textbf{Search} & \textbf{Math} \\
        \midrule
        \multirow{2}{*}{Qwen3-8B}
            & Placement (s)  & 0.036 & 0.037 & 0.036 \\
            & Resource manager (s)       & 5.69 & 5.27 & 5.06 \\
        \midrule
        \multirow{2}{*}{Qwen3-14B}
            & Placement (s)  & 0.038 & 0.037 & 0.036 \\
            & Resource manager (s)       & 5.05 & 4.99 & 4.97 \\
        \midrule
        \multirow{2}{*}{Qwen3-32B}
            & Placement (s)  & 0.037 & 0.038 & 0.036 \\
            & Resource manager (s)       & 5.06 & 5.01 & 4.98 \\
        \bottomrule
    \end{tabular}
    }
    \vspace{0.1in}
    \caption{Algorithm overheads in \sysname.}
    \vspace{-0.4in}
    \label{tab:eval:algorithm-overhead}
\end{table}
\section{Related Work}
\label{sec:related}

\paraf{LLM Training.}
LLM development bifurcates into pre-training and post-training phases with distinct computational
demands. Pre-training systems~\cite{zhang2025disttrain,ge2025bytescale,jiang2024megascale} optimize static throughput via multidimensional parallelism
to maximize FLOP utilization on massive corpora.
In contrast, post-training paradigms introduce dynamic workloads that alternate between
multi-step trajectory generation (inference) and policy optimization (training).
To support this, existing frameworks adopt either \emph{colocated} or
\emph{disaggregated} architectures. Colocated systems~\cite{verl,zhong2025optimizing,hu2025taming} execute both
generation and training on the same worker pool. 
Disaggregated systems~\cite{zhong2025streamrl,slime_github,fu2025areal} physically separate inference and training
onto dedicated and even heterogeneous worker pools.
\sysname is orthogonal to these existing RL frameworks and can be seamlessly integrated to provide
efficient agentic trajectory rollout.

\parabf{LLM Inference.}
Extensive LLM inference research focuses on optimizing LLM serving and RL rollout efficiency.
In scheduling, frameworks~\cite{wu2023fast,yu2022orca,luo2025autellix} employ iterative scheduling and selective batching to
mitigate head-of-line blocking. For memory management, many systems~\cite{kwon2023efficient,jin2024ragcache,qin2024mooncake,yao2025cacheblend,wu2025tokenlake} leverage
PagedAttention~\cite{kwon2023efficient} to eliminate memory fragmentation and facilitate prefix caching.
Architecturally, DistServe~\cite{zhong2024distserve} and MegaScale-Infer~\cite{zhu2025megascale}
disaggregate compute resources to isolate interference between distinct computation phases.
\sysname is orthogonal to these foundational techniques and incorporates
these step-level optimizations while introducing trajectory-level metadata
to orchestrate efficient agentic rollout.

\section{Conclusion}
\label{sec:conclusion}

We present \sysname, a trajectory-centric framework that optimizes the agentic RL
rollout phase by mitigating the straggler effect caused by long-tailed, multi-step agentic trajectories.
By decomposing the long-tail trajectory time into queueing, interference, and base per-token time,
\sysname employs a synergistic orchestration of trajectory-centric scheduling,
placement, and resource management. Our evaluation demonstrates
that \sysname effectively neutralizes long-tail bottlenecks, achieving up to 2.5$\times$ higher
throughput than state-of-the-art baselines.

\def\UrlBreaks{\do\/\do-}
\bibliographystyle{ACM-Reference-Format}
\bibliography{xin}

\clearpage

\end{sloppypar}
\end{document}